\documentclass{article}

\PassOptionsToPackage{numbers}{natbib}
\usepackage[preprint]{neurips_2026}

\usepackage[utf8]{inputenc}
\usepackage[T1]{fontenc}
\usepackage{float}
\usepackage{hyperref}
\hypersetup{hidelinks}
\usepackage{url}
\usepackage{booktabs}
\usepackage{multirow}
\usepackage{amsmath}
\usepackage{amsthm}
\usepackage{amsfonts}
\usepackage{amssymb}
\usepackage{nicefrac}
\usepackage{microtype}
\usepackage{xcolor}
\usepackage{colortbl}
\usepackage{graphicx}
\usepackage{wrapfig}
\usepackage{needspace}
\usepackage[most]{tcolorbox}
\usepackage{enumitem}

\definecolor{promptbg}{RGB}{248,250,252}
\definecolor{promptborder}{RGB}{58,87,130}

\newcommand{\acgreedy}{Answer-Coverage Greedy}
\newcommand{\scgreedy}{Support-Coverage Greedy}
\newcommand{\readout}{Latent Verification Debate}
\newcommand{\secref}[1]{Sec.~\ref{#1}}
\newcommand{\figref}[1]{Fig.~\ref{#1}}
\newcommand{\tabref}[1]{Table~\ref{#1}}
\newcommand{\appref}[1]{Appendix~\ref{#1}}
\newcommand{\eqnref}[1]{Eqn.~\eqref{#1}}

\newcommand{\lemref}[1]{Lemma~\ref{#1}}
\newcommand{\propref}[1]{Proposition~\ref{#1}}

\newtcolorbox{promptbox}[1]{
    breakable,
    colback=promptbg,
    colframe=promptborder,
    colbacktitle=promptborder,
    coltitle=white,
    title={#1},
    fonttitle=\small\bfseries\sffamily,
    boxrule=0.6pt,
    arc=2pt,
    left=7pt,
    right=7pt,
    top=6pt,
    bottom=6pt,
    before skip=0.6em,
    after skip=0.8em
}

\usepackage{amsmath,amsfonts,bm,eqnarray}
\usepackage{cancel}
\usepackage{amsthm}
\usepackage{tikz}
\usetikzlibrary{tikzmark}

\newtheorem{lemma}{Lemma}

\newtheorem{proposition}{Proposition}

\newtheorem{innercustomgeneric}{\customgenericname}
\providecommand{\customgenericname}{}
\newcommand{\newcustomtheorem}[2]{%
  \newenvironment{#1}[1]
  {%
   \renewcommand\customgenericname{#2}%
   \renewcommand\theinnercustomgeneric{##1}%
   \innercustomgeneric
  }
  {\endinnercustomgeneric}
}

\newcustomtheorem{customThm}{Theorem}
\newcustomtheorem{customLemma}{Lemma}
\newcustomtheorem{customCor}{Corollary}
\newcustomtheorem{customProposition}{Proposition}

\def\figref#1{Fig.~\ref{#1}}

\def\secref#1{Sec.~\ref{#1}}
\def\appref#1{Appendix~\ref{#1}}

\def\eqref#1{equation~\ref{#1}}

\def\eqnref#1{Eqn.~\ref{#1}}

\def\1{\bm{1}}

\DeclareMathAlphabet{\mathsfit}{\encodingdefault}{\sfdefault}{m}{sl}
\SetMathAlphabet{\mathsfit}{bold}{\encodingdefault}{\sfdefault}{bx}{n}

\def\gP{{\mathcal{P}}}

\renewcommand{\frac}{\tfrac}

\title{When Debate Helps:
Proposal Supply and\\Verification-Aware Readout in Multi-Agent Reasoning}

\author{
  \textbf{Zihao Zhao}$^{1,*}$ \quad
  \textbf{Tunyu Zhang}$^{2}$ \quad
  \textbf{Haizhou Shi}$^{2}$
  \\[0.35em]
  \textbf{Yusong Zhao}$^{1}$ \quad
  \textbf{Xinxi Zhang}$^{2}$ \quad
  \textbf{Hao Wang}$^{1,*}$
  \\
  $^{1}$University of Illinois Urbana-Champaign \quad
  $^{2}$Rutgers University
  \\
  $^{*}$Corresponding authors: \texttt{\{zihaoz,hwangml\}@illinois.edu}
}

\begin{document}

\maketitle

\begin{abstract}
Multi-agent debate can improve reasoning, yet often fails to beat simple
majority voting. We argue that successful debate requires two distinct
mechanisms: proposal supply must surface a correct answer, and readout must
identify that answer when voting misses it. We formalize the first requirement
through recoverable headroom, which measures cases where a correct proposal is
available but the majority answer is wrong. For the second, we develop
\readout{} (LVD), an accounting model in which candidate proposals receive
answer-specific verification evidence before final generation. Controlled
fixed-proposal interventions estimate this latent effect in equivalent
peer-support units and show that correct evidence changes answer probabilities
and generated decisions while proposal supply remains fixed. To
improve proposal supply, we construct societies from neural-thicket agents
using labeled and label-free coverage objectives. Across two backbones and
matched-budget reasoning benchmarks, coverage-selected societies increase
complementary proposal supply and improve aggregate accuracy in repeated
stochastic evaluations. Round-level controls further show that interaction
provides gains beyond applying the same finalizer directly to the initial
proposals. These results identify proposal coverage and truth-sensitive
evidence use as complementary conditions for debate to outperform voting.
Code is available at
\url{https://github.com/Wang-ML-Lab/when-debate-helps}.
\end{abstract}

\section{Introduction}

Debate lets language-model agents propose, critique, and aggregate alternative
reasoning paths~\citep{du2024improving, chanchateval, xiong2023examining, estornellacc}. However, it often fails to beat
simple majority voting or self-consistency~\citep{wang2023selfconsistency, wang2024rethinking}. This raises a basic question: when does interaction help beyond
voting?

A prior martingale-null analysis~\citep{choi2026debate} gives a useful baseline: when interaction is
symmetric and truth-agnostic, debate can mix proposals without creating positive
expected drift toward correctness, and therefore is not guaranteed to improve
over majority vote.
We use this null as a boundary case: practical debate can succeed only when it
adds structure outside neutral proposal mixing.

We argue that two mechanisms determine when this happens. The
\textbf{proposal} mechanism determines which candidate answers the agent society
makes available. The \textbf{readout} mechanism distills these candidates into a
final answer. For debate to help, the proposal process must surface at least one
correct answer, and the readout must select correct proposals rather than merely
count how many agents repeat each proposal. When the correct proposal is held by a
minority and vote misses it, this becomes a headroom-recovery problem: the
proposal stage has hit the answer, but the majority readout has not
(\secref{sec:headroom_lvd}).
This view
connects debate to verifiers, LLM judges, critique, and candidate reranking~\citep{cobbe2021trainingverifierssolvemath, lightman2023letsverifystepstep, li-etal-2023-making, zheng2023judgingllmasajudgemtbenchchatbot, NEURIPS2023_91edff07}, but treats
verification as part of the debate readout rather than only as an external
selector.

\begin{figure}[t]
\centering
\includegraphics[width=0.92\linewidth]{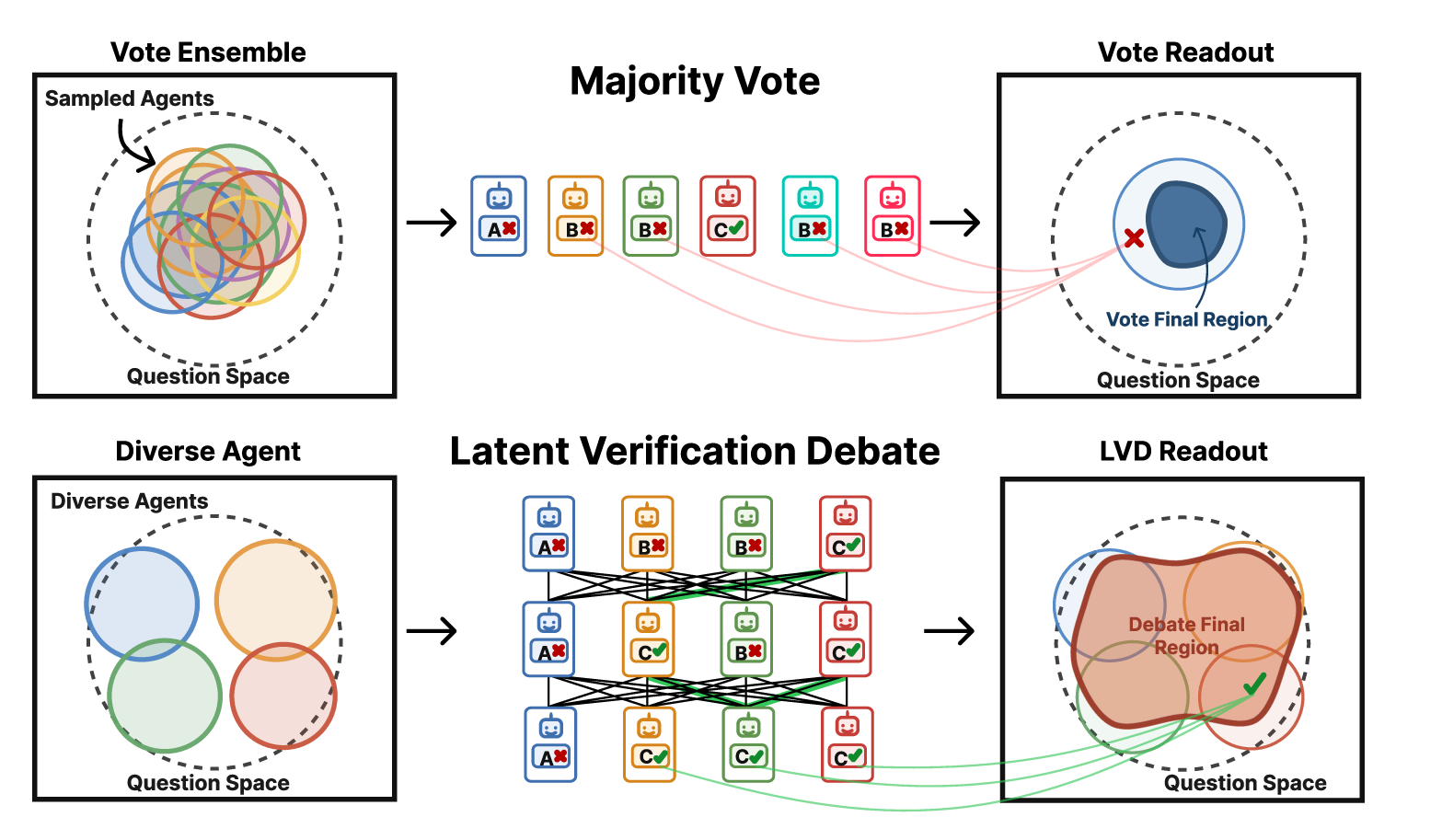}
\vskip -0.5cm
\caption{Overview of the supply-readout view. Transparent regions denote
questions for which individual sampled agents can surface the correct answer;
solid regions denote questions recovered by the final readout. \textbf{Top:} Overlapping
sampled agents provide redundant proposal coverage, so majority vote keeps a
small shared region and can reinforce a repeated wrong answer. \textbf{Bottom:} Diverse
agents provide complementary proposal coverage; after debate, LVD uses
verification-aware evidence to recover a surfaced correct minority answer.
}
\label{fig:teaser}
\vskip -4mm
\end{figure}

We formalize this framework through recoverable headroom: cases where the
proposal stage surfaces the correct answer but majority vote selects a wrong
answer. Here, a proposal hit means that at least one surfaced answer matches the
correct answer. For readout, \readout{} (LVD) models candidate-specific
verification evidence added to proposal-count dynamics. For supply, we sample
neural-thicket agents by applying small perturbations to a pretrained model's
weights~\citep{gan2026neuralthicketsdiversetask, zhang2026tokur,
shi2025trainingfree} and use labeled and label-free coverage objectives to
select agents with complementary proposal behavior. \figref{fig:teaser}
summarizes the supply-readout framework.

Across reasoning and multi-discipline benchmarks, we test this framework on
Qwen3-4B and OLMo3-7B under matched inference budgets. Coverage-based
neural-thicket construction tests the proposal-supply mechanism, while
controlled fixed-proposal interventions test whether peer evidence changes
readout without changing the candidate set. Repeated stochastic evaluations
and round-level controls assess robustness and separate immediate finalizer
effects from gains produced by interaction.
Our contributions are as follows:
\begin{itemize}[nosep]
\item We introduce recoverable headroom and an exact decomposition that
separates recovered headroom, newly correct cases, and damage to vote-correct
cases.
\item We develop \readout{} (LVD) as a verification-aware readout model and
introduce controlled fixed-proposal probes that operationalize its latent
evidence weights without treating them as literal hidden counters.
\item We introduce labeled and label-free coverage-based construction for
neural-thicket agent pools and show, through matched stochastic replications
and mechanism diagnostics, when complementary proposal supply yields gains over
Vanilla Debate.
\end{itemize}

\section{Setup and Motivating Observations}
\label{sec:setup_motivation}
\label{sec:setup}

We first define the proposal objects shared by vote, debate, and proposal-hit
diagnostics, then use two short observations to motivate the theory.

\subsection{Proposal Setup and the Martingale Null}
\label{subsec:martingale_null}
\label{subsec:dcm_accounting}

We consider reasoning tasks where responses can be compared through a final
answer value. Let $\mathcal X$ be the input space and let $\mathcal Y$ be the
decision space. For an input $x\in\mathcal X$ with correct answer
$y^\star\in\mathcal Y$, let $\mathcal P=\{e_1,\ldots,e_N\}$ be a candidate
agent pool and let $S\subseteq\mathcal P$ with $|S|=K$ be the selected society.
We use round $t=0$ for the initial independent proposals and rounds $t\ge 1$ for
subsequent debate updates. At round $t$, agent $e_i$ produces a response
containing a rationale and a normalized answer
\[
    o_{i,t}(x)=\bigl(r_{i,t}(x),a_{i,t}(x)\bigr),
    \qquad a_{i,t}(x)\in\mathcal Y.
\]
The rationale $r_{i,t}$ is visible textual evidence. The answer value $a_{i,t}$
is used for vote and proposal-hit diagnostics.
Let $\mathcal M_t(x,S)$ be the multiset of answer values produced by the
selected society at round $t$. For each candidate answer $k$, define its
proposal count as
\[
    C_t^{(k)}=\sum\nolimits_{e_i\in S}\mathbf 1\{a_{i,t}(x)=k\}.
\]
Since each selected agent contributes one answer at round $t$,
$\sum_{k\in\mathcal Y}C_t^{(k)}=K$.
Majority vote operates on $\mathcal M_t$:
\[
    y_t^{\mathrm{vote}}(x,S)
    \in
    \arg\max_{k\in\mathcal Y} C_t^{(k)}.
\]
Ties are resolved by a fixed tie-breaking rule. The proposal-hit diagnostic asks
whether the correct answer was surfaced, that is, whether
$y^\star$ appears in $\mathcal M_t(x,S)$. This diagnostic is not an available readout. It
measures whether the proposal supply contains a recoverable correct answer.

In multi-agent debate, agents update from visible prior responses. Let
$\mathcal N(i)\subseteq S$ be the agents whose previous responses are visible to
$e_i$. At round $t\ge 1$,
agent $e_i$ observes
\[
    \mathcal F_{i,t}(x,S)=\{o_{j,t-1}(x):e_j\in\mathcal N(i)\}
\]
and updates through a one-round debate operator,
\[
    o_{i,t}(x)=\mathcal D_i(x,\mathcal F_{i,t}(x,S)).
\]
Here $\mathcal F_{i,t}$ is the local context used by agent $e_i$ for its
round-$t$ update. We use $\mathcal F_t$ to denote the global visible debate history
up to round $t$.

To connect this setup to prior debate-vs-vote theory~\citep{choi2026debate, smit2023are, zhang2025stopovervaluingmultiagentdebate, kaesberg-etal-2025-voting},
we use the Dirichlet-categorical-multinomial (DCM) accounting layer as a
proposal-count baseline. In this view, a count-based
belief state $\alpha_t$ assigns mass to candidate answers. The initial state
$\alpha_0$ represents the model's pre-debate answer belief. Observing proposal
counts updates this state symmetrically:
\begin{equation}
    \alpha_{t+1}^{(k)}=\alpha_t^{(k)}+C_t^{(k)}.
    \label{eq:dcm_count_update}
\end{equation}
This update records how often answer $k$ is surfaced; it does not know whether
$k$ is correct.

\textbf{Symmetric Truth-Agnostic Null.}
Under symmetric, truth-agnostic interaction, agents mix surfaced proposals
without a signal that preferentially favors the true answer. Prior work
formalizes this setting as a martingale over aggregate correctness mass~\citep{choi2026debate}:
\begin{equation}
    \mathbb E[Q_{t+1}\mid\mathcal F_t]=Q_t.
    \label{eq:martingale_null}
\end{equation}
Here $Q_t$ denotes the aggregate correctness mass after round $t$. Thus
interaction alone does not create positive expected drift toward correctness.
Equivalently, debate that only mixes proposals is not guaranteed to improve over
majority vote.

\textbf{Observation 1: Debate Dynamics Can Depart from the Null.}
Round-wise debate diagnostics, reported in
\figref{fig:motivating_observations}(a) and detailed in
\appref{app:motivating_observations}, Qwen2.5-3B-Instruct~\cite{qwen2025qwen25technicalreport} shows no
positive drift after interaction, which is empirically consistent with
null-like proposal mixing. Under the same protocol, Qwen3-4B-Instruct~\citep{yang2025qwen3technicalreport} shows
persistent gain over vote. We treat this as an empirical boundary: weaker
dynamics can look like neutral mixing, while practical debate can contain
structure beyond symmetric proposal mixing.

\begin{figure}[t]
\centering
\includegraphics[width=\linewidth]{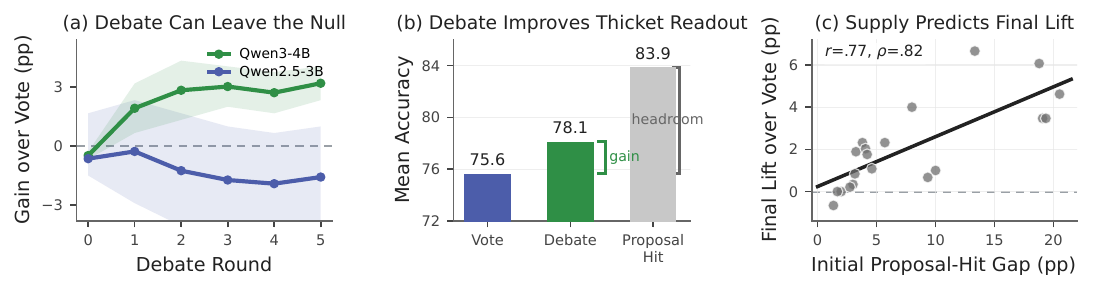}
\vskip -0.4cm
\caption{
Motivating observations.
\textbf{(a)} Qwen2.5-3B-Instruct shows null-like debate dynamics with little gain over
vote, while Qwen3-4B-Instruct shows positive gain across rounds.
\textbf{(b)} Under the same neural-thicket supply regime, \emph{Debate} readout improves over
\emph{vote}, while the proposal-hit upper bound remains higher, indicating recoverable
headroom.
\textbf{(c)} Final debate lift is positively correlated with initial proposal-hit
headroom, suggesting that debate is supply-constrained.
Together, these observations motivate a supply-readout account: debate helps when correct minority proposals are surfaced and the readout can recover them through verification-aware dynamics.
}
\label{fig:motivating_observations}
\vskip -0.4cm
\end{figure}

\textbf{Observation 2: Neural-Thicket Societies Surface Recoverable Expertise.}
A second observation comes from neural-thicket agent societies~\citep{gan2026neuralthicketsdiversetask}.
A \underline{\emph{neural thicket}} is a local pool of models obtained by perturbing pretrained
weights $\theta_0$ near a base model, so that each sampled perturbation can
behave as a nearby specialist $e_m$.
We use such perturbations as heterogeneous debate agents. Concretely, for base
weights $\theta_0$, seed $s_m$, and noise scale $\sigma_m$, a perturbed agent has weights
$
\theta_m=\theta_0+\sigma_m\epsilon(s_m),
$.
The language-model agent $e_m$ uses the perturbed weights $\theta_m$. Here
$\epsilon(s_m)$ is a standard Gaussian noise vector drawn from
$\mathcal N(0,I)$ using seed $s_m$; $\sigma_m$ controls the perturbation
scale.
The thicket pool is
$
\gP_N(\theta_0)=\{e_m\}_{m=1}^N
$.
To form a small society, the standard neural-thicket baseline assigns each
sampled agent a scalar score $v_m$ on a training set
$\mathcal D_{\mathrm{train}}$.
Typically, $v_m$ is training accuracy (defined in \appref{app:seed_profiling_details}).
The baseline selects the top-$K$ agents by $v_m$ and uses them for thicket vote
or debate. Our method keeps the same sample-and-score interface but replaces
scalar top-$K$ selection with coverage-based construction (\secref{sec:supply}).

These thicket societies allow us to separate proposal supply from final readout.
\figref{fig:motivating_observations}(b) fixes the selected society and compares
round-0 vote, final debate readout, and the round-0 proposal-hit upper bound.
Debate improves over vote under this fixed supply, so the gain is not merely a
larger agent pool or more independent samples. The remaining proposal-hit gap
shows that some surfaced correct answers are still missed by the final readout.
\figref{fig:motivating_observations}(c) further shows that round-0 proposal-hit
headroom correlates with final lift over vote. This pattern suggests that
debate gains are supply-constrained: the readout mostly converts surfaced or
early-emerging correct proposals that majority vote does not extract reliably.

\textbf{Takeaway.}
The martingale null is a useful boundary case, but not a
complete description of practical debate. Under the same supply regime, debate
can act as a stronger readout than majority vote; across settings, its lift is
tied to round-0 proposal-hit headroom. This motivates our supply-readout view:
useful debate requires enough proposal supply to create recoverable headroom
and a verification-aware readout to convert that headroom into final accuracy.
\secref{sec:headroom_lvd} formalizes this view.

\section{Headroom Recovery Theory}
\label{sec:headroom_lvd}

The observations above separate debate into proposal supply and readout. We now
formalize this separation using round $0$ as the independent proposal stage
before interaction. For a selected society $S$, the round-0 answers define the
initial proposal supply, the vote baseline, and the proposal-hit diagnostic.
Majority vote is strong because it uses the most common surfaced answer, but it
can discard a correct minority proposal. We measure this missed opportunity with
recoverable headroom.

\subsection{Recoverable Headroom and Gain}
\label{subsec:headroom_accounting}
\label{subsec:recovery_damage_decomposition}

For a fixed selected society $S$, let
$\mathcal M(x,S):=\mathcal M_0(x,S)$ be the round-0 proposal multiset used for
headroom accounting. For each instance $x$, define two indicators:
\begin{equation}
\begin{aligned}
    V(x,S)&=\mathbf 1\{\operatorname{Vote}(\mathcal M(x,S))=y^\star\},\\
    P(x,S)&=\mathbf 1\{\exists a\in\mathcal M(x,S):a=y^\star\}.
\end{aligned}
    \label{eq:vote_proposal_hit}
\end{equation}
Here $V$ records whether majority vote is correct, while $P$ records whether
the proposal multiset contains the correct answer. We define society-level
recoverable headroom by averaging over the evaluation distribution:
\begin{equation}
    \mathcal H(S)=\Pr_x(P(x,S)=1,V(x,S)=0).
    \label{eq:headroom}
\end{equation}
Thus $\mathcal H(S)$ measures cases where the society contains the correct
answer but the majority readout misses it.
Let $D_\eta(x,S)$ be the correctness indicator of a final readout with
configuration $\eta$. We define the readout's recovery, new-correctness, and
damage rates by
\begin{equation}
\begin{aligned}
    R_\eta(S)&=\Pr(D_\eta=1\mid P=1,V=0),\\
    I_\eta(S)&=\Pr(D_\eta=1\mid P=0,V=0),\\
    L_\eta(S)&=\Pr(D_\eta=0\mid V=1).
\end{aligned}
    \label{eq:recovery_damage}
\end{equation}
Here $R_\eta$ measures recovery of vote-missed surfaced answers, $I_\eta$
measures correctness that emerges after the initial proposal stage, and
$L_\eta$ measures harm to vote-correct cases. The net gain over vote is
\begin{equation}
    \mathcal G_\eta(S)=\mathbb E[D_\eta(x,S)-V(x,S)].
    \label{eq:gain_def}
\end{equation}
\begin{proposition}[\textbf{Recovery-Damage Gain Decomposition}]
\label{prop:recovery_damage_decomposition}
With the initial proposal supply fixed, the gain over vote decomposes as
\begin{equation}
    \mathcal G_\eta(S)
    =
    \mathcal H(S)R_\eta(S)
    +
    \Pr(P=0,V=0)I_\eta(S)
    -
    \Pr(V=1)L_\eta(S).
    \label{eq:headroom_decomposition}
\end{equation}
\end{proposition}

The proof is in \appref{app:proof_recovery_damage_decomposition}.
This identity is the main accounting result. A readout beats vote when recovered
headroom plus new-correctness exceeds damage to vote-correct cases. Our main
mechanism claim concerns the first term: LVD should improve the conversion of
surfaced correct minorities into final answers. The second term is reported
separately because it captures correctness not explained by the initial
proposal-hit supply.

\subsection{LVD as Verification-Aware Readout}
\label{subsec:lvd_model}
\label{subsec:lvd_recovery_condition}

\secref{subsec:martingale_null} introduced the symmetric truth-agnostic null.
We now make the corresponding proposal-count baseline explicit with a
Dirichlet-categorical-multinomial (DCM) accounting layer~\citep{choi2026debate}. Note that we use this layer as readout-level bookkeeping over surfaced answer
counts, not as a full behavioral model of each debate agent.

More broadly, this probabilistic readout is inspired by the hierarchical
Bayesian deep learning perspective~\citep{BDL,BDLSurvey,VDI}, in which observed neural signals are coupled with latent probabilistic variables through
Bayesian inference. Here, we apply this perspective at the multi-agent
readout level, treating surfaced proposals and verification evidence as
observations that update a latent distribution over candidate answers.

At round $t$, let $\mathcal Y_t$ be the finite candidate answer set considered by the readout.
For multiple-choice tasks, $\mathcal Y_t$ is the task-defined option set; for
exact-answer tasks, it is the set of surfaced normalized answer clusters in the
visible transcript $\mathcal F_t$. Let $C_t^{(k)}$ be the count of visible
proposals whose normalized answer is $k$, let
$C_t^+=\sum_{k\in\mathcal Y_t}C_t^{(k)}$, and let
$q_\theta^0(k\mid x)$ be the readout model $\theta$'s pre-debate answer prior over
$\mathcal Y_t$. A complete single-round readout with one valid answer from each
of $K$ agents has $C_t^+=K$. With prior mass $\kappa>0$, the proposal-only
posterior and predictive are
\begin{equation}
    \alpha_t^{\mathrm{prop}}(k)
    =
    \kappa q_\theta^0(k\mid x)+C_t^{(k)},
    \qquad
    p_{\mathrm{prop}}(y=k\mid x,\mathcal F_t)
    =
    \frac{\alpha_t^{\mathrm{prop}}(k)}{\kappa+C_t^+}.
    \label{eq:prop_alpha}
\end{equation}
This proposal-only readout treats surfaced answers symmetrically: an answer
gains influence only by appearing more often in the proposal multiset.

Our proposed LVD augments this proposal-count baseline with an effective verified
pseudo-count vector $W_{\eta,t}$~\cite{VIR}. Each $W_{\eta,t}^{(k)}\ge 0$ is the additional
verification-aware evidence assigned to answer $k$ by readout configuration
$\eta$; $W_{\eta,t}$ is a theoretical answer-level abstraction, not an observed
hidden state inside the LLM. The LVD posterior is (see the conjugacy argument in
\appref{app:proof_verification_weighted_conjugacy}).
\begin{equation}
\begin{aligned}
    \alpha_{\mathrm{LVD},\eta}^{(k)}
    &=\kappa q_\theta^0(k\mid x)+C_t^{(k)}+W_{\eta,t}^{(k)},\\
    p_{\mathrm{LVD},\eta}(y=k\mid x,\mathcal F_t,W_{\eta,t})
    &=\frac{\kappa q_\theta^0(k\mid x)+C_t^{(k)}+W_{\eta,t}^{(k)}}
    {\kappa+C_t^+ + W_{\eta,t}^+},
    \qquad
    W_{\eta,t}^+=\sum\nolimits_{k\in\mathcal Y_t}W_{\eta,t}^{(k)}.
    \label{eq:lvd_posterior}
\end{aligned}
\end{equation}
Following the hierarchical Bayesian deep learning perspective
~\citep{BDL,BDLSurvey,NPN}, we can further introduce a latent verification
variable and marginalize it to obtain the effective verification evidence.
One such latent-verification
parameterization, proved in \appref{app:proof_lvd_marginalization}, is
$W_{\eta,t}^{(k)}=\tau_{\eta,t}q_\eta^v(k\mid x,\mathcal F_t)$, where
$q_\eta^v$ is a soft allocation of verification support and $\tau_{\eta,t}$ is
the total effective verification mass. The proposal-only readout is the special
case $W_{\eta,t}^{(k)}=0$ for all $k$.

\begin{proposition}[\textbf{Verification-Directed Positive Drift}]
\label{prop:verification_directed_drift}
Fix $x$ and $\mathcal F_t$, let $y^\star\in\mathcal Y_t$ be the correct answer,
and define $A_t=\kappa+C_t^+$ and
$\alpha_t^\star=\alpha_t^{\mathrm{prop}}(y^\star)$. For a fixed effective
pseudo-count vector with $W_{\eta,t}^+>0$, let
$Q_t^{\mathrm{prop}}:=\alpha_t^\star/A_t$. LVD increases posterior mass on
$y^\star$ relative to the proposal-count baseline if and only if
\begin{equation}
    \frac{W_{\eta,t}^{(y^\star)}}{W_{\eta,t}^+}
    >
    \frac{\alpha_t^\star}{A_t}
    =
    Q_t^{\mathrm{prop}}.
    \label{eq:lvd_positive_drift_ratio}
\end{equation}
\end{proposition}

The proof is in \appref{app:proof_verification_directed_drift}.
This is the readout-level analogue of the martingale-null comparison. Symmetric
proposal mixing gives no positive expected drift in correctness mass; LVD
creates non-null drift only when verified pseudo-counts allocate more mass to
$y^\star$ than the proposal-only posterior already does.

\begin{proposition}[\textbf{Correct-Minority Recovery}]
\label{prop:correct_minority_recovery}
On a headroom instance, suppose the proposal-count winner is an incorrect answer
$y_v\ne y^\star$, while $y^\star\in\mathcal Y_t$. LVD ranks the correct minority
above the proposal-count winner if and only if
\begin{equation}
    W_{\eta,t}^{(y^\star)}-W_{\eta,t}^{(y_v)}
    >
    \alpha_t^{\mathrm{prop}}(y_v)-\alpha_t^{\mathrm{prop}}(y^\star).
    \label{eq:lvd_recovery_condition}
\end{equation}
\end{proposition}

The proof is in \appref{app:proof_correct_minority_recovery}.
The right side is the proposal-layer disadvantage of the correct minority,
including both raw proposal counts and the base prior. The left side is its
verified-evidence advantage. Thus LVD recovers a correct minority when
verification-aware evidence overcomes the proposal-layer advantage of the wrong
winner. This does not make debate universally helpful: gains still require
recoverable headroom and enough recovery to outweigh damage to vote-correct
cases.

\section{Theory-Guided System Design}
\label{sec:supply}

\propref{prop:recovery_damage_decomposition}$\sim$\ref{prop:verification_directed_drift} suggest two design components: (1) We should increase proposal-hit supply,
\[
    \mathcal P_{\mathrm{hit}}(S)=\Pr(P(x,S)=1)=\Pr(V=1)+\mathcal H(S),
\]
which can improve vote accuracy and create more vote-missed correct proposals
for the readout to recover. (2) When a correct candidate is surfaced, the readout
should assign it larger verified evidence; otherwise extra supply need not
improve final accuracy. We therefore use coverage-based society construction
for supply and a verification-aware readout configuration.

\subsection{Coverage-Based Society Construction}
\label{subsec:thicket_pool}

The construction step chooses which sampled thicket agents enter the society.
Instead of keeping the $K$ agents with the largest individual training-set
scores $v_m$ from \eqnref{eq:thicket_eval}, we choose a small society
$S\subseteq\gP_N(\theta_0)$ whose members cover complementary proposal hits.
For a pool \(\gP_N\) and budget $K$, the ideal supply target is
\begin{equation}
    \mathcal P_{\mathrm{hit},K}^\star(\gP_N)
    =
    \max\nolimits_{S\subseteq \gP_N,\ |S|=K}\Pr(P(x,S)=1).
    \label{eq:proposal_hit_potential}
\end{equation}
We approximate this target on the construction split, either with labels or with
an agreement proxy.

\subsection{Answer-Coverage and Support-Coverage Greedy}
\label{subsec:proxy_greedy_covering}

Directly optimizing proposal-hit supply on the held-out distribution is
unavailable during construction because it depends on labels. We instead use a
construction-set surrogate. Let
$\mathcal U_{\mathrm{cal}}=\mathcal D_{\mathrm{train}}$ be the construction
split, let $\mathcal P=\gP_N(\theta_0)$ be the sampled agent pool, and let
$a_e(u)$ be agent $e$'s normalized initial answer on question $u$. For a proxy
score $\psi(e,u)\in[0,1]$, the selected society maximizes
\begin{equation}
\begin{aligned}
    S_K &\in \arg\max\nolimits_{S\subseteq\mathcal P,\ |S|=K}
    \widehat{\mathcal C}_\psi(S),\\
    \widehat{\mathcal C}_\psi(S)
    &=\frac{1}{|\mathcal U_{\mathrm{cal}}|}
    \sum\nolimits_{u\in\mathcal U_{\mathrm{cal}}}\max\nolimits_{e\in S}\psi(e,u).
    \label{eq:coverage_objective}
\end{aligned}
\end{equation}
The max makes this a question-level coverage objective: a question contributes
once any selected agent receives a high proxy score. Exact maximization is a
cardinality-constrained maximum-coverage problem, so we use the standard greedy
approximation~\citep{nemhauser1978analysis}. Starting from
$S_0=\emptyset$, step $m$ chooses
\begin{equation}
    e_m\in\arg\max\nolimits_{e\in\mathcal P\setminus S_{m-1}}
    \left[\widehat{\mathcal C}_\psi(S_{m-1}\cup\{e\})
    -\widehat{\mathcal C}_\psi(S_{m-1})\right],
    \label{eq:greedy_step}
\end{equation}
and sets $S_m=S_{m-1}\cup\{e_m\}$ until $m=K$.

\textbf{Answer-Coverage Greedy (AC-Greedy).}
When labels are available, AC-Greedy uses
\begin{equation}
    \psi_{\mathrm{AC}}(e,u)=\mathbf 1\{a_e(u)=y^\star(u)\}.
    \label{eq:answer_coverage_proxy}
\end{equation}
It chooses agents by marginal increase in labeled answer coverage on the
construction split and serves as a label-based construction reference.

\textbf{Support-Coverage Greedy (SC-Greedy).}
SC-Greedy removes the need for labels by computing coverage from agent agreement. For each
construction question $u$, we cluster valid extracted answers. Let $c_e(u)$ be
the cluster containing agent $e$'s answer, and let $\widehat p_u(c)$ be the
fraction of valid extracted answers in cluster $c$:
\[
    \widehat p_u(c)
    =
    \frac{
    |\{e\in\mathcal P:c_e(u)=c\}|
    }{
    |\{e\in\mathcal P:e\text{ has a valid extracted answer on }u\}|
    }.
\]
If no agent has a valid extracted answer on $u$, all support scores for that
question are set to zero. The label-free proxy is
\begin{equation}
    \psi_{\mathrm{SC}}(e,u)=
    \begin{cases}
    \widehat p_u(c_e(u)), & \text{if }e\text{ has a valid extracted answer},\\
    0, & \text{otherwise.}
    \end{cases}
    \label{eq:support_coverage_proxy}
\end{equation}
This score is not a majority-only indicator: it can
still select complementary agents.

\textbf{Verification-Aware Readout.}
We pair the coverage-selected societies with a fixed debate instruction and a
vote-then-judge finalizer. \secref{sec:experiments} tests the observable
predictions of LVD through controlled interventions that hold proposal supply
fixed.

\begin{table*}[t]
    \caption{Accuracy (\%) on Qwen3-4B-Instruct and OLMo-3-7B-Instruct under a unified computation budget. SC and Neural thicket vote rows use \texttt{vote20}; debate rows use round-3 (R3) majority vote from matched five-round, four-agent traces. MMLU-Redux is split into five 60-example subjects. Avg is the unweighted mean over the eight displayed columns, and Avg Ranking is the mean within-block method rank. \textbf{Boldface} and \underline{underline} denote the best and second-best result per column within each method block and model. Gray rows mark our proposed agent pool construction methods.}
    \vspace{0.15em}
    \label{tab:main_results}
    \label{tab:call_aligned_r3}
    \label{tab:olmo_main_r3_call_aligned}
    \centering
    \scriptsize
    \resizebox{1\linewidth}{!}{%
    \setlength{\tabcolsep}{4pt}
    \begin{tabular}{l ccc ccccc cc}
    \toprule
      \multicolumn{1}{c}{\multirow{2}{*}{\textbf{Method}}}
      & \multicolumn{3}{c}{\textbf{Reasoning}}
      & \multicolumn{5}{c}{\textbf{MMLU-Redux}}
      & \multirow{2}{*}{\textbf{Avg}}
      & \multirow{2}{*}{\textbf{Avg Ranking}} \\
      \cmidrule(lr){2-4} \cmidrule(lr){5-9}
      & \textbf{AMC12}
      & \textbf{MATH500}
      & \textbf{GPQA}
      & \textbf{Moral}
      & \textbf{Chem.}
      & \textbf{CS}
      & \textbf{Acct.}
      & \textbf{Med.}
      & & \\
    \midrule
    \rowcolor{black!12}
    \multicolumn{11}{c}{\textbf{Backbone: Qwen3-4B-Instruct}} \\
    \midrule
    \multicolumn{11}{@{}c@{}}{\color{black!45}\hrulefill\quad\textit{Label-Free Methods}\quad\hrulefill} \\
    SC@20
        & 91.20
        & \underline{92.33}
        & 49.13
        & 30.00
        & \underline{70.00}
        & \textbf{81.67}
        & \underline{75.00}
        & \underline{83.33}
        & 71.58
        & 2.38 \\
    Vanilla Debate
        & \textbf{92.69}
        & 91.67
        & \textbf{52.31}
        & \underline{50.00}
        & \textbf{71.67}
        & \underline{80.00}
        & \textbf{76.67}
        & \underline{83.33}
        & \underline{74.79}
        & \textbf{1.81} \\
    \rowcolor{black!6}
    SC-Greedy Debate (Ours)
        & \underline{92.29}
        & \textbf{92.67}
        & \underline{51.45}
        & \textbf{63.33}
        & 68.33
        & \textbf{81.67}
        & 73.33
        & \textbf{86.67}
        & \textbf{76.22}
        & \textbf{1.81} \\
    \midrule
    \multicolumn{11}{@{}c@{}}{\color{black!45}\hrulefill\quad\textit{Label-Based Methods}\quad\hrulefill} \\
    Neural Thicket Vote
        & 90.93
        & \textbf{91.33}
        & 50.00
        & \underline{70.00}
        & 63.33
        & \underline{80.00}
        & \textbf{76.67}
        & \underline{81.67}
        & 75.49
        & \underline{2.19} \\
    Top-4 Debate
        & \underline{92.15}
        & \underline{90.67}
        & \underline{52.31}
        & 68.33
        & \underline{66.67}
        & 78.33
        & \underline{75.00}
        & \textbf{88.33}
        & \underline{76.47}
        & 2.25 \\
    \rowcolor{black!6}
    AC-Greedy Debate (Ours)
        & \textbf{92.29}
        & \textbf{91.33}
        & \textbf{52.60}
        & \textbf{73.33}
        & \textbf{68.33}
        & \textbf{85.00}
        & 73.33
        & 80.00
        & \textbf{77.03}
        & \textbf{1.56} \\
    \midrule
    \rowcolor{black!12}
    \multicolumn{11}{c}{\textbf{Backbone: OLMo-3-7B-Instruct}} \\
    \midrule
    \multicolumn{11}{@{}c@{}}{\color{black!45}\hrulefill\quad\textit{Label-Free Methods}\quad\hrulefill} \\
    SC@20
        & \textbf{87.01}
        & \textbf{88.33}
        & 35.55
        & \underline{25.00}
        & \underline{50.00}
        & 61.67
        & \underline{56.67}
        & 61.67
        & 58.24
        & 2.12 \\
    Vanilla Debate
        & \underline{86.33}
        & \underline{85.33}
        & \textbf{39.31}
        & 23.33
        & 48.33
        & \underline{68.33}
        & \textbf{66.67}
        & \underline{63.33}
        & \underline{60.12}
        & \underline{2.00} \\
    \rowcolor{black!6}
    SC-Greedy Debate (Ours)
        & 84.03
        & 84.00
        & \underline{37.57}
        & \textbf{28.33}
        & \textbf{56.67}
        & \textbf{76.67}
        & 55.00
        & \textbf{70.00}
        & \textbf{61.53}
        & \textbf{1.88} \\
    \midrule
    \multicolumn{11}{@{}c@{}}{\color{black!45}\hrulefill\quad\textit{Label-Based Methods}\quad\hrulefill} \\
    Neural Thicket Vote
        & \textbf{86.87}
        & \underline{84.67}
        & \textbf{37.57}
        & 28.33
        & 51.67
        & \textbf{78.33}
        & \underline{46.67}
        & \textbf{71.67}
        & 60.72
        & \textbf{2.00} \\
    Top-4 Debate
        & \underline{85.93}
        & 84.33
        & \underline{36.42}
        & \textbf{36.67}
        & \underline{56.67}
        & \textbf{78.33}
        & \underline{46.67}
        & \textbf{71.67}
        & \underline{62.09}
        & \textbf{2.00} \\
    \rowcolor{black!6}
    AC-Greedy Debate (Ours)
        & 85.25
        & \textbf{85.00}
        & 35.84
        & \underline{35.00}
        & \textbf{58.33}
        & \underline{73.33}
        & \textbf{56.67}
        & \textbf{71.67}
        & \textbf{62.64}
        & \textbf{2.00} \\
    \bottomrule
    \end{tabular}
    }
    \vskip -0.2cm
\end{table*}

\section{Experiments}
\label{sec:experiments}

We test whether the supply-readout view predicts when debate improves over
vote. We evaluate two main debate backbones, Qwen3-4B-Instruct and
OLMo-3-7B-Instruct~\citep{yang2025qwen3technicalreport, olmo2026olmo3}, on mathematical reasoning (AMC12, MATH500)~\citep{hendrycks2021measuringmathematicalproblemsolving}, scientific
question answering (GPQA)~\citep{rein2024gpqa}, and multi-discipline reasoning (MMLU-Redux)~\citep{hendrycks2021measuring, gema-etal-2025-done}. Each
debate run uses $K=4$ agents and $T=5$ rounds. We compare \acgreedy{} (\textbf{AC-Greedy}) and
\scgreedy{} (\textbf{SC-Greedy}) against baselines including Self-Consistency (\textbf{SC})~\citep{wang2023selfconsistency}, \textbf{Vanilla Debate}~\citep{du2024improving}, \textbf{Neural Thicket Vote}, and
\textbf{Neural Thicket Debate}~\citep{gan2026neuralthicketsdiversetask}
under matched vote/debate budgets. Full details on
dataset splits, thicket-pool construction, decoding parameters, token budgets,
prompts, context visibility, and tie-breaking are in
Appendices~\ref{app:eval_details}--\ref{app:prompts}; complete round-level
statistics are in \appref{app:round_level_statistics}.

\textbf{Diverse Expert Supply for Debate Systems.}
\tabref{tab:main_results} asks whether our proposed coverage-selected thicket agents form
stronger debate societies than vote-only scaling, Vanilla Debate, and top-score
thicket baselines. We use a unified computation budget: vote rows use
\texttt{vote20}, while debate rows use the round-3 (R3) majority state from matched
four-agent traces. Across different backbones, our \scgreedy{} debate has the
best average within the label-free block, and our \acgreedy{} debate has the best
average within the label-based block. The gains are clearest on the
MMLU-Redux split, where coverage-based selection improves several subjects even
when the aggregate average is close.

\textbf{Stochastic Robustness.}
Table~\ref{tab:main_results} retains the submitted method-specific decoding
configuration and its fixed global seed. In a separate robustness experiment,
we rerun Vanilla Debate, SC-Greedy, and AC-Greedy with a common stochastic
configuration (temperature $0.7$, top-$p=0.95$) across five matched generation
seeds. Relative to Vanilla Debate, SC-Greedy and AC-Greedy improve the
eight-column average by $+1.14$ points (paired 95\% CI $[0.60,1.67]$) and
$+1.46$ points ($[0.01,2.91]$) on Qwen3-4B, and by $+0.96$ points
($[0.02,1.90]$) and $+2.01$ points ($[1.32,2.69]$) on OLMo3-7B. These
intervals quantify generation-seed variability rather than uncertainty from
resampling test questions; complete per-benchmark results and paired tests are
reported in Appendix~\ref{app:stochastic_robustness}.

\textbf{Can Coverage-Based Selection Transfer Out of Domain?}
\phantomsection\label{subsec:ood_transfer}
We next ask whether coverage-selected societies only improve the domain used
for selection, or whether their complementarity transfers out of domain. We
select agents on AMC12 and evaluate the same selected societies on MMLU-Redux
without reselecting agents. This setting tests whether diverse proposal supply
selected from one reasoning domain remains useful under a multi-discipline
target shift.

\begin{table*}[t]
    \caption{AMC12-selected seed transfer to MMLU-Redux. Init, R3, and R5 report
    pre-debate, round-3, and round-5 metrics over the same four agents. Final is
    the R5 vote-then-judge readout; Lift is Final minus Init Vote. Values are
    percentages except $N$. Gray rows mark our greedy construction methods.}
    \vspace{0.15em}
    \label{tab:amc12_mmlu_transfer_r5}
    \centering
    \scriptsize
    \resizebox{1\linewidth}{!}{
    \setlength{\tabcolsep}{4pt}
    \begin{tabular}{lrrrrrrrrr}
    \toprule
      \textbf{Method}
      & \textbf{N}
      & \textbf{Init Vote}
      & \textbf{R3 Vote}
      & \textbf{R5 Vote}
      & \textbf{Final}
      & \textbf{Init Hit}
      & \textbf{R3 Hit}
      & \textbf{R5 Hit}
      & \textbf{Lift} \\
    \midrule
    \rowcolor{black!12}
    \multicolumn{10}{c}{\textbf{Backbone: Qwen3-4B-Instruct, AMC12 $\rightarrow$ MMLU-Redux}} \\
    \midrule
    \rowcolor{black!6}
    SC-Greedy Debate (Ours)
        & 300 & 72.33 & \underline{75.67} & \underline{76.33} & \underline{76.33}
        & \textbf{84.00} & \underline{77.67} & 76.33 & \textbf{+4.00} \\
    Top-4 Debate
        & 300 & \underline{72.67} & 75.00 & 74.67 & 75.33
        & \textbf{84.00} & \underline{77.67} & \underline{76.67} & \underline{+2.67} \\
    \rowcolor{black!6}
    AC-Greedy Debate (Ours)
        & 300 & \textbf{75.33} & \textbf{77.33} & \textbf{77.00} & \textbf{76.67}
        & \underline{83.67} & \textbf{78.00} & \textbf{78.00} & +1.33 \\
    \midrule
    \rowcolor{black!12}
    \multicolumn{10}{c}{\textbf{Backbone: OLMo-3-7B-Instruct, AMC12 $\rightarrow$ MMLU-Redux}} \\
    \midrule
    \rowcolor{black!6}
    SC-Greedy Debate (Ours)
        & 300 & \underline{55.33} & \textbf{56.33} & \textbf{58.33} & \textbf{58.33}
        & \textbf{70.00} & \textbf{59.67} & \textbf{59.33} & \textbf{+3.00} \\
    Top-4 Debate
        & 300 & \textbf{56.00} & \textbf{56.33} & 56.33 & \underline{56.67}
        & 67.00 & 57.33 & 57.00 & +0.67 \\
    \rowcolor{black!6}
    AC-Greedy Debate (Ours)
        & 300 & 54.00 & \textbf{56.33} & \underline{56.67} & 56.33
        & \underline{69.33} & \underline{58.00} & \underline{57.33} & \underline{+2.33} \\
    \bottomrule
    \end{tabular}
    }
    \vskip -0.1cm
\end{table*}

\tabref{tab:amc12_mmlu_transfer_r5} shows positive transfer for the proposed
coverage-selected societies, but the stronger construction varies by backbone.
On Qwen3-4B-Instruct, our \acgreedy{} achieves the best final accuracy, while
our \scgreedy{} achieves the largest improvement over the initial vote. On
OLMo-3-7B-Instruct, \scgreedy{} achieves both the best final accuracy and the
largest improvement. This suggests that coverage selection can capture reusable agent
complementarity, although transfer is not uniform across backbones. The
headroom, recovery, and damage decomposition for these transfer runs is in
Appendix~\ref{app:secondary_construction_diagnostics}.

\begin{figure}[t]
\centering
\includegraphics[width=\textwidth]{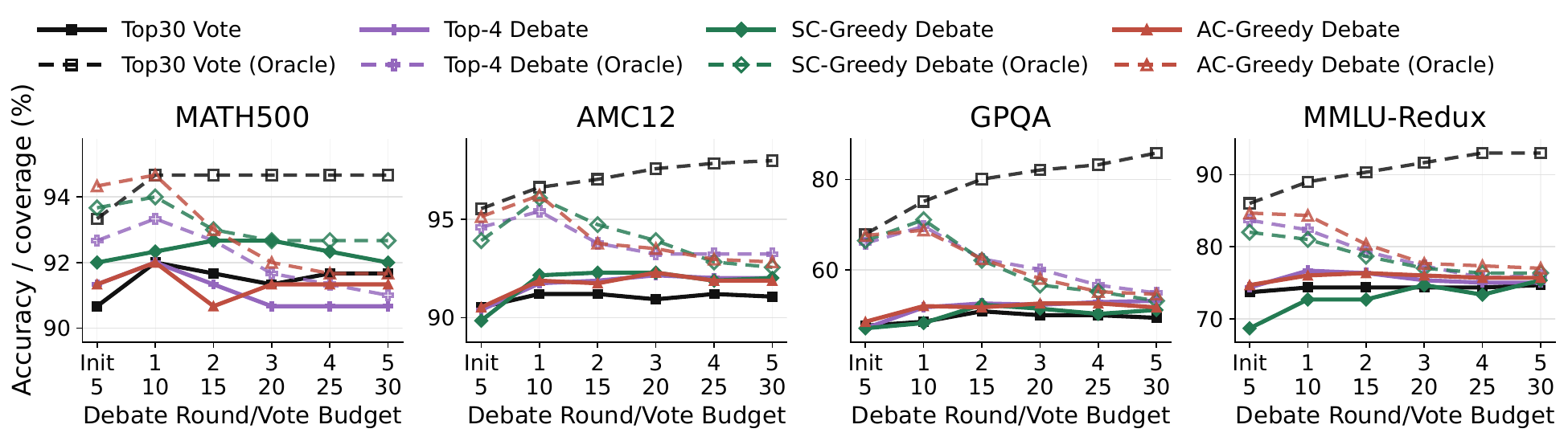}
\caption{Label-based five-round dynamics. Solid lines show round-level accuracy before finalization; dashed hollow legend entries
marked ``Oracle'' show proposal-hit coverage. Later rounds can improve vote
accuracy while reducing answer diversity, so final gains depend on recovering
correct proposals before coverage collapses. The label-free counterpart is in
\figref{fig:r5_debate_dynamics_label_free}.}
\label{fig:r5_debate_dynamics_label_based}
\vskip -0.2cm
\end{figure}

\textbf{Dynamics and Mechanism Analysis.}
\figref{fig:r5_debate_dynamics_label_based} shows that early rounds can contain
more proposal-hit coverage than majority vote uses, while later rounds increase
consensus but reduce the alternatives. To isolate interaction, we apply the
same vote-then-judge finalizer independently at every saved round. Across four
societies and five benchmarks, R0 accuracy rises only from 76.09\% to 76.45\%,
whereas the R2 and R3 states reach 78.34\% and 78.52\%---1.89 and 2.07 points
beyond the R0 readout. An immediate finalizer therefore cannot explain the
R2--R3 plateau. On OLMo MATH500, all four societies retain 3.33--6.33 points
of headroom but have negative final lift, identifying a conversion bottleneck.
Appendix~\ref{app:dynamics_conversion_controls} gives both controls, and
\appref{app:round_level_statistics} gives complete recovery and damage results.

\phantomsection
\label{subsec:controlled_readout_probes}
\textbf{Controlled Fixed-Proposal Readout Probes.}
We directly test whether evidence assigned to surfaced candidates changes
answer selection while proposal supply remains fixed. For each
backbone--dataset pair, we uniformly sample 50 traces at random from the saved
traces whose initial proposal set contains at least one correct and one
incorrect answer. Within each trace, we hold the question, target agent,
initial response, prompt style, and candidate set fixed. We vary only the peer
evidence supplied to the target agent.

We first estimate the effect of an original correct peer response in
equivalent-support units. Let $N$ denote a neutral peer context, $G_d$ a context
containing $d\in\{1,2,3\}$ short messages supporting the gold candidate $y_g$,
and $T$ the original context containing a correct peer response from the saved
trace. For a designated wrong candidate $y_w$, we measure
\[
m(x)=\log p(y_g\mid x)-\log p(y_w\mid x).
\]
The controlled shift is $D_d=m(G_d)-m(N)$, and the original trace-response
shift is $\Delta m_{\mathrm{trace}}=m(T)-m(N)$. For each recovery case, where
the target initially answers incorrectly while a correct peer proposal is
available, we interpolate the observed $d=0,1,2,3$ dose-response curve to find
$\widehat W$, the controlled-support dose that matches
$\Delta m_{\mathrm{trace}}$. We treat $\widehat W$ as an operational
equivalent-support measure, not a literal internal counter.

We next test whether the effect depends on evidential content rather than
repetition of an answer. From the saved traces, we construct matched correct
and wrong peer rationales, remove all stated answers and option references, and
match each pair in length. We then compare the gold-versus-other-answer margin
after the correct rationale, $T$, and the wrong rationale, $F$. A positive
$T-F$ means that the answer-stripped correct reasoning leaves the model more
favorable to the gold answer than equally long wrong reasoning.

Finally, we test whether these probability shifts alter generated decisions.
From the same fixed mixed-proposal states, we generate fresh answers after
$d\in\{1,2,3\}$ peer messages under correct-support and matched control
conditions. The controls provide no candidate evidence, withhold the supported
answer, or insert an answer-stripped rationale from another question.
\tabref{tab:lvd_controlled_probes} reports the original trace
effect, equivalent support, answer-stripped correct-over-wrong gap, and
free-generation accuracy gain at $d=3$.

\begin{table*}[t]
\centering
\caption{Controlled fixed-proposal evidence for verification-aware readout.
Trace effect is the answer-margin shift from an original correct peer response
relative to neutral context; $\widehat W$ is its median equivalent controlled
support over recovery cases. $T-F$ is the gold-margin difference between
answer-stripped correct and wrong rationales. $\Delta$Acc is correct-support
accuracy minus neutral accuracy at $d=3$. All entries use the main readout
configuration; secondary configuration controls are reported in the appendix.}
\label{tab:lvd_controlled_probes}
\scriptsize
\setlength{\tabcolsep}{4pt}
\resizebox{\textwidth}{!}{%
\begin{tabular}{llcccc}
\toprule
\textbf{Backbone} & \textbf{Dataset}
& $\boldsymbol{\Delta m_{\mathrm{trace}}}$ \textbf{[95\% CI]}
& \textbf{Median $\boldsymbol{\widehat W}$}
& \textbf{Answer-stripped $\boldsymbol{T-F}$}
& $\boldsymbol{\Delta}$\textbf{Acc at $\boldsymbol{d=3}$} \\
\midrule
Qwen3-4B & GPQA       & 8.15 [4.96, 11.41]  & 2.23 & 4.59 & +.14 \\
Qwen3-4B & MMLU-Redux & 11.40 [9.02, 13.97] & $>3$ & 5.50 & +.04 \\
OLMo3-7B & GPQA       & 4.57 [3.65, 5.50]   & 1.92 & 2.01 & +.56 \\
OLMo3-7B & MMLU-Redux & 5.47 [4.52, 6.43]   & $>3$ & 2.10 & +.58 \\
\bottomrule
\end{tabular}%
}
\end{table*}

Original correct peer responses have positive margin effects in all four
backbone--dataset combinations. Their median effects correspond to roughly two
controlled messages on GPQA and more than three on MMLU-Redux. The
answer-stripped $T-F$ gap is also positive in every backbone--dataset
combination. At $d=3$, correct support improves accuracy over neutral context
in all four combinations; averaged across them, accuracy rises from .325 to
.655. Because the proposal set never changes, these effects isolate
truth-sensitive readout rather than proposal supply. Full intervention tables
and configuration controls are in \appref{app:controlled_readout_probes} and
\appref{app:prompt_control_tables}.

\section{Related Work}
\label{sec:related_work}

\textbf{Debate, Vote, and Null Models.}
Multi-agent debate asks several language-model agents to propose, critique, and
revise answers before a final decision~\citep{du2024improving, chanchateval, xiong2023examining, chen2024reconcile}. Empirical
results vary across tasks and aggregation rules~\citep{smit2024should,wang2024rethinking,estornell2024multi,zhang2025if}, and majority vote remains a
strong baseline~\citep{wang2023selfconsistency,kaesberg-etal-2025-voting}. Recent theory
formalizes a symmetric, truth-agnostic martingale null for debate-like
interaction~\citep{choi2026debate}. We use that result as a proposal-count
baseline and show how verification-directed pseudo-counts can create
readout-level positive drift.

\textbf{Verification and Judging.}
Our readout model connects to work on verifiers, LLM-as-a-judge methods~\citep{tan2025judgebenchbenchmarkevaluatingllmbased, zhou2025evaluatingjudgesevaluatorsjetts},
critique, and reranking~\citep{cobbe2021trainingverifierssolvemath, lightman2023letsverifystepstep, li-etal-2023-making, zheng2023judgingllmasajudgemtbenchchatbot, NEURIPS2023_91edff07}.
OrchRM uses execution artifacts to learn rewards for orchestrator training and
test-time scaling~\citep{tsang2026rewardmodelingmultiagentorchestration}.
LVD places answer-specific verification within debate, where it changes
proposal pseudo-counts before final generation.
Our probabilistic treatment also draws on hierarchical Bayesian deep
learning~\citep{BDL,BDLSurvey}, while placing the latent probabilistic
structure at the multi-agent readout level.

\textbf{Neural Thickets and Coverage Selection.}
Neural thickets show that nearby perturbations of a base model can form an
agent society (pool) with heterogeneous strengths~\citep{gan2026neuralthicketsdiversetask,qiu2026evolutionstrategiesscalellm}. We use
this pool as the proposal-supply side of debate. Our coverage constructor
selects complementary specialists from the thicket rather than ranking agents
only by individual average accuracy. This connects to coverage objectives for
budgeted subset selection~\citep{nemhauser1978analysis}.
Our coverage signal can be supervised, using true correctness on a construction
set, or label-free, using answer-cluster support.

\section{Conclusion}

Debate helps when proposal supply and readout succeed together. Coverage-based neural-thicket construction increases complementary correct proposals, while LVD uses answer-specific evidence to improve their readout. Controlled experiments confirm truth-sensitive readout effects, and matched stochastic evaluations show gains over Vanilla Debate across both backbones, though headroom conversion remains task- and backbone-dependent.

These results do not imply that debate is generally better than voting. Interaction can provide additional gains but may also discard useful alternatives or damage vote-correct cases. Its practical value therefore depends on proposal headroom, the backbone's ability to evaluate competing evidence, and inference cost. See
\appref{appsec:limitation} for detailed limitations.

\begin{ack}
We thank all reviewers, the area chair, and the senior area chair for their
valuable comments.

\textbf{Funding.}
HW is supported by the Amazon Faculty Research Award, Microsoft AI \& Society
Fellowship, Nokia Faculty Research Award, NSF CAREER Award IIS-2340125,
NIH grant R01CA297832, and NSF grant IIS-2127918.

\textbf{Competing interests.}
The authors declare no competing interests.
\end{ack}

\bibliographystyle{abbrv}
\bibliography{references,debate}

\appendix

\section{Additional Details}

\subsection{Implementation Details for \figref{fig:motivating_observations}}
\label{app:motivating_observations}

Panel (a) uses four-agent, five-round \texttt{standard} Vanilla Debate and scores
every saved round with the same \texttt{vote\_then\_judge\_tiebreak} readout.
For Qwen2.5-3B-Instruct, the curve averages over AMC12, MATH500, GPQA, and
MMLU-Redux. For Qwen3-4B-Instruct, the curve uses the
matched MMLU-Redux and GPQA diagnostics. The plotted value is the per-round
readout accuracy minus the R0 majority-vote accuracy, in percentage points.

Panel (b) reports the unweighted average over AMC12, MATH500, GPQA, and
MMLU-Redux for the label-free thicket society. The Vote bar is R0
majority-vote accuracy, the Debate readout bar is the final five-round debate
accuracy, and the Proposal-hit bar is the R0 proposal-hit@4 rate for the same
selected society.

Panel (c) uses the matched Qwen3-4B-Instruct five-round cells for sample,
answer-coverage, support-coverage, and top-accuracy societies across the same
five benchmarks. The x-axis is $P_0 - V_0$, where $P_0$ is the R0 proposal-hit
rate and $V_0$ is the R0 majority-vote accuracy. The y-axis is $F - V_0$, where
$F$ is the final debate readout accuracy. The reported correlations are Pearson
$r=0.77$ and Spearman $\rho=0.82$ across the 20 cells.

\subsection{Evaluation Details}
\label{app:eval_details}

The main runs use data-specific generation budgets: 8192 tokens for AMC12 and
MATH500, and 4096 tokens for GPQA and MMLU-Redux. We keep
decoding fixed within each matched comparison. Self-consistency and Vanilla
Debate sample the base model with temperature 0.7 and top-$p$ 0.95. Thicket
agents and the tie-break judge use deterministic decoding with temperature 0
and top-$p$ 1.0. All runs use global seed 42; agent-level sampling seeds are
deterministic offsets from that seed.

The main experiments use the verification debate prompt. Each debate-round
prompt shows an agent its own previous response and the clipped previous-round
responses from the other agents, asks it to verify the peer answers, and then
asks for a new final answer in the dataset's required format. The standard
prompt-control ablation instead asks agents to compare peer solutions and
revise if needed. The tie-break judge prompt sees the final candidate responses
and is used only when the final vote has no unique winner.

\subsection{Compute Resources}
\label{app:compute_resources}

All experiments are inference-only: we do not train or fine-tune either
backbone. We use BF16 inference with vLLM on an internal cluster. The reported
runs used NVIDIA A100 GPUs with 80~GB memory and NVIDIA RTX PRO 6000 Blackwell
Server Edition GPUs with 96~GB memory. Each downstream evaluation or diagnostic
uses one GPU, while seed profiling distributes independent perturbed models
over 6--8 GPUs. Table~\ref{tab:compute_resources} reports observed runtime from
successful-run logs; GPU-hours are the number of allocated GPUs multiplied by
wall-clock time.

\begin{table}[H]
\centering
\caption{Compute used by the reported experiment families. Profiling times vary
with dataset size and output length; downstream times vary with the backbone,
dataset, and number of debate rounds.}
\label{tab:compute_resources}
\small
\begin{tabular}{p{0.34\textwidth}p{0.18\textwidth}p{0.39\textwidth}}
\toprule
Experiment family & Workers per run & Observed wall time and compute \\
\midrule
Qwen seed profiling (500 agents per dataset)
    & 6--8 GPUs
    & 3.3--18.4 hours per dataset; approximately 27--110 GPU-hours \\
OLMo seed profiling (300 retained agents per dataset)
    & 8 GPUs
    & 1.3--6.6 hours per dataset; about 10--53 GPU-hours \\
Five-round downstream and prompt-control evaluation
    & 1 GPU per job
    & 0.2--4.9 hours per job \\
Five-run stochastic robustness experiment
    & 1 GPU per job; 120 jobs
    & 156.8 GPU-hours in total \\
Controlled mechanism, readout, and transfer diagnostics
    & 1 GPU per job
    & Within the same 0.2--4.9-hour single-job envelope \\
\bottomrule
\end{tabular}
\end{table}

Summing successful-run wall times and GPU allocations gives an estimated
\textbf{0.9k GPU-hours} for the experiments reported in this paper, with a
plausible range of 0.8--1.0k GPU-hours. The range reflects early Qwen profiling
jobs that were paused and resumed before scheduler-level accounting was
retained; the 156.8 GPU-hours for the stochastic robustness experiment are
measured directly from complete logs. Queue time and CPU-only analysis are not
included. Retained run artifacts occupy approximately 12~GB, excluding cached
model weights. The broader research project used additional compute for
preliminary sweeps, failed or resumed jobs, and experiments on backbones not
reported here; that exploratory compute is not included in this estimate.

\subsection{Existing-Asset Licenses}
\label{app:asset_licenses}

Table~\ref{tab:asset_licenses} records the versions and terms of the model and
dataset assets used in the paper. We use these assets only for research
evaluation and do not redistribute model weights or raw benchmark questions.

\begin{table}[H]
\centering
\caption{Licenses and stated terms for existing model and dataset assets.}
\label{tab:asset_licenses}
\footnotesize
\begin{tabular}{p{0.31\textwidth}p{0.20\textwidth}p{0.40\textwidth}}
\toprule
Asset and version & License or terms & Use in this work \\
\midrule
\href{https://huggingface.co/Qwen/Qwen3-4B-Instruct-2507}{Qwen3-4B-Instruct-2507}
    & Apache-2.0 & Main 4B backbone; research evaluation only \\
\href{https://huggingface.co/allenai/Olmo-3-7B-Instruct}{OLMo-3-7B-Instruct}
    & Apache-2.0; Ai2 Responsible Use Guidelines
    & Main 7B backbone; research evaluation only \\
\href{https://huggingface.co/Qwen/Qwen2.5-3B-Instruct}{Qwen2.5-3B-Instruct}
    & Qwen Research License
    & Motivating experiment; non-commercial research evaluation \\
\href{https://huggingface.co/datasets/Idavidrein/gpqa}{GPQA}
    & CC BY 4.0 & Gated benchmark; examples are not redistributed \\
\href{https://huggingface.co/datasets/edinburgh-dawg/mmlu-redux-2.0}{MMLU-Redux 2.0}
    & CC BY 4.0 & Benchmark evaluation \\
\href{https://huggingface.co/datasets/Hothan/OlympiadBench}{OlympiadBench}
    & Apache-2.0 & Benchmark evaluation \\
\href{https://huggingface.co/datasets/HuggingFaceH4/MATH-500}{MATH-500}
    & Mirror declares no license; source MATH and PRM800K repositories are MIT
    & We use the published split without redistributing it \\
\href{https://huggingface.co/datasets/edev2000/amc12-full}{AMC12-full}
    & No open license declared; original problems are copyright MAA
    & Dataset card states research/educational use; questions are not redistributed \\
\bottomrule
\end{tabular}
\end{table}

\subsection{Five-Run Stochastic Robustness}
\label{app:stochastic_robustness}
We repeat the three main-table methods involved in the coverage comparison---
Vanilla Debate, SC-Greedy, and AC-Greedy---for five matched generation runs on
both backbones and all eight benchmark columns. The original full test sets,
model, society, prompts, and finalizer remain fixed. For this robustness
experiment only, all three methods use temperature $0.7$ and top-$p=0.95$; the
generation seed varies across runs and is shared across methods within each
run. This common stochastic configuration differs from the submitted
configuration in Table~\ref{tab:main_results}, where thicket agents use
deterministic decoding and all methods use one fixed global seed.

Table~\ref{tab:stochastic_robustness} reports mean $\pm$ standard deviation
across the five runs. We compute the eight-column average within each run
before summarizing across runs. Paired 95\% confidence intervals and two-sided
$p$-values use paired $t$-tests on these five run-level averages. Thus, the
intervals measure sensitivity to generation randomness under stochastic
decoding; they do not include uncertainty from resampling test questions.

\begin{table*}[t]
\centering
\caption{Five-run stochastic robustness under a common decoding configuration.
Panel (a) reports mean $\pm$ SD across five generation seeds; bold marks the
highest mean within each backbone and column, not statistical significance.
Panel (b) reports paired comparisons of the five matched run-level
eight-column averages. Effects and confidence intervals are in percentage
points. Comparisons near the nominal $p=.05$ threshold should be interpreted
cautiously.}
\label{tab:stochastic_robustness}
\vspace{0.3em}
\textbf{(a) Per-benchmark accuracy and eight-column average.}\\[0.25em]
\tiny
\resizebox{\textwidth}{!}{%
\begin{tabular}{llrrrrrrrrr}
\toprule
Backbone & Method & AMC12 & MATH500 & GPQA & Chem. & CS & Med. & Moral & Acct. & 8-col. Avg \\
\midrule
\multirow{3}{*}{Qwen3-4B}
 & Vanilla Debate & \textbf{92.56 $\pm$ 0.21} & 91.40 $\pm$ 0.55 & 50.40 $\pm$ 1.01 & \textbf{71.00 $\pm$ 3.65} & 81.00 $\pm$ 1.49 & 83.33 $\pm$ 0.00 & 67.67 $\pm$ 1.49 & 76.33 $\pm$ 1.39 & 76.71 $\pm$ 1.10 \\
 & SC-Greedy & \textbf{92.56 $\pm$ 0.17} & \textbf{92.27 $\pm$ 0.43} & 51.62 $\pm$ 0.48 & 70.33 $\pm$ 3.80 & 84.00 $\pm$ 0.91 & \textbf{86.00 $\pm$ 0.91} & 69.67 $\pm$ 1.39 & 76.33 $\pm$ 2.17 & 77.85 $\pm$ 1.19 \\
 & AC-Greedy & 92.21 $\pm$ 0.23 & 91.93 $\pm$ 0.49 & \textbf{52.20 $\pm$ 0.90} & 70.00 $\pm$ 1.18 & \textbf{84.33 $\pm$ 2.24} & 84.67 $\pm$ 1.39 & \textbf{73.00 $\pm$ 1.39} & \textbf{77.00 $\pm$ 1.39} & \textbf{78.17 $\pm$ 1.08} \\
\midrule
\multirow{3}{*}{OLMo3-7B}
 & Vanilla Debate & \textbf{86.82 $\pm$ 0.68} & 85.80 $\pm$ 1.02 & 35.90 $\pm$ 0.24 & 60.67 $\pm$ 2.24 & 71.67 $\pm$ 4.08 & 72.67 $\pm$ 0.91 & 28.67 $\pm$ 3.80 & 57.00 $\pm$ 1.83 & 62.40 $\pm$ 1.72 \\
 & SC-Greedy & 86.60 $\pm$ 0.59 & 85.60 $\pm$ 0.86 & \textbf{36.99 $\pm$ 1.78} & 59.00 $\pm$ 1.90 & 74.33 $\pm$ 2.24 & 74.33 $\pm$ 0.91 & 32.67 $\pm$ 4.35 & 57.33 $\pm$ 1.90 & 63.36 $\pm$ 1.69 \\
 & AC-Greedy & 86.39 $\pm$ 0.89 & \textbf{86.13 $\pm$ 0.38} & 36.71 $\pm$ 0.87 & \textbf{62.33 $\pm$ 1.49} & \textbf{77.33 $\pm$ 1.90} & \textbf{75.00 $\pm$ 1.18} & \textbf{33.00 $\pm$ 3.61} & \textbf{58.33 $\pm$ 2.04} & \textbf{64.40 $\pm$ 1.39} \\
\bottomrule
\end{tabular}%
}

\vspace{0.7em}
\textbf{(b) Paired comparisons of run-level averages.}\\[0.25em]
\scriptsize
\begin{tabular}{llrrr}
\toprule
Backbone & Comparison & $\Delta$ Avg & Paired 95\% CI & $p$-value \\
\midrule
\multirow{3}{*}{Qwen3-4B}
 & SC-Greedy $-$ Vanilla Debate & $+1.135$ & $[0.599, 1.671]$ & $0.0042$ \\
 & AC-Greedy $-$ Vanilla Debate & $+1.455$ & $[0.005, 2.905]$ & $0.0495$ \\
 & AC-Greedy $-$ SC-Greedy & $+0.320$ & $[-1.225, 1.865]$ & $0.5961$ \\
\midrule
\multirow{3}{*}{OLMo3-7B}
 & SC-Greedy $-$ Vanilla Debate & $+0.960$ & $[0.022, 1.898]$ & $0.0468$ \\
 & AC-Greedy $-$ Vanilla Debate & $+2.005$ & $[1.320, 2.691]$ & $0.0012$ \\
 & AC-Greedy $-$ SC-Greedy & $+1.045$ & $[0.209, 1.881]$ & $0.0256$ \\
\bottomrule
\end{tabular}
\end{table*}

Both coverage-selected methods have higher aggregate accuracy than Vanilla
Debate on both backbones in these matched runs. All four comparisons against
Vanilla Debate have nominal $p<.05$, although the Qwen AC-Greedy interval and
the OLMo SC-Greedy interval lie close to zero. AC-Greedy and SC-Greedy are not
distinguishable on Qwen3-4B, while AC-Greedy has the higher run-level average
on OLMo3-7B.

\subsection{Dataset and Split Details}
\label{app:dataset_split_details}

The main call-aligned table evaluates AMC12, MATH500, GPQA, and MMLU-Redux.
AMC12 uses 360 training questions for construction and 739 held-out test
questions. MATH500 uses 200 training questions and 300 held-out test questions.
GPQA uses 200 training questions and 346 held-out test questions. MMLU-Redux
uses five subjects, with 40 training questions and 60 held-out test questions
per subject, for 300 test questions in total. The five-round dynamics in the
diagnostic appendix also include OlympiadBench, with 200 training questions and
474 held-out test questions.

\subsection{Seed Profiling Details}
\label{app:seed_profiling_details}

Table~\ref{tab:seed_profiling_details} records the perturbation scales used to
construct the main-table thicket pools. Qwen3-4B-Instruct uses a 500-seed
legacy pool. A fixed global random seed sampled both the perturbation seeds and
their scales, so the realized counts across the three retained scales are 160,
174, and 166. OLMo samples 100 seeds at each listed scale. Its diagnostic sweeps
also included smaller and larger scales, but the final downstream experiments
use the middle three scales: $5\times 10^{-4}$, $10^{-3}$, and
$2\times 10^{-3}$.
For top-$K$ thicket baselines, the scalar score is training accuracy:
\begin{equation}
    v_m
    =
    \frac{1}{|\mathcal D_{\mathrm{train}}|}
    \sum\nolimits_{(x,y^\star)\in\mathcal D_{\mathrm{train}}}
    \mathbf 1\{\operatorname{ans}(e_m,x)=y^\star\},
    \label{eq:thicket_eval}
\end{equation}
where $\operatorname{ans}(e_m,x)$ is the normalized answer from agent $e_m$.

\begin{table}[t]
\centering
\caption{Seed-quality diagnostics for the main-table backbones. Values are
percentages on the profiling split. Base ref. is the same-split unperturbed
pass@1 reference. The Used column marks scales retained for downstream
construction.}
\scriptsize
\resizebox{\textwidth}{!}{
\begin{tabular}{llrrrrrrr}
\toprule
Backbone & Profiling split & $\sigma$ & Seeds & Base ref. & Seeds $>$ base & Mean & Best & Used \\
\midrule
Qwen3-4B & MATH500 first200 & 0.0005 & 160 & 90.00 & 59/160 & 89.82 & 92.50 & yes \\
Qwen3-4B & MATH500 first200 & 0.001 & 174 & 90.00 & 51/174 & 89.58 & 92.50 & yes \\
Qwen3-4B & MATH500 first200 & 0.002 & 166 & 90.00 & 11/166 & 88.10 & 91.00 & yes \\
Qwen3-4B & AMC12 train360 & 0.0005 & 160 & 84.17 & 5/160 & 82.40 & 85.83 & yes \\
Qwen3-4B & AMC12 train360 & 0.001 & 174 & 84.17 & 7/174 & 81.96 & 86.11 & yes \\
Qwen3-4B & AMC12 train360 & 0.002 & 166 & 84.17 & 0/166 & 78.97 & 83.33 & yes \\
Qwen3-4B & GPQA train200 & 0.0005 & 160 & 48.00 & 79/160 & 48.13 & 53.00 & yes \\
Qwen3-4B & GPQA train200 & 0.001 & 174 & 48.00 & 69/174 & 47.73 & 54.50 & yes \\
Qwen3-4B & GPQA train200 & 0.002 & 166 & 48.00 & 13/166 & 44.64 & 52.00 & yes \\
Qwen3-4B & MMLU-Redux train200 & 0.0005 & 160 & 71.00 & 32/160 & 70.00 & 74.00 & yes \\
Qwen3-4B & MMLU-Redux train200 & 0.001 & 174 & 71.00 & 23/174 & 68.92 & 76.50 & yes \\
Qwen3-4B & MMLU-Redux train200 & 0.002 & 166 & 71.00 & 28/166 & 67.47 & 75.00 & yes \\
\midrule
OLMo-3-7B & MATH500 random200 & 0.0005 & 100 & 78.50 & 91/100 & 80.22 & 82.50 & yes \\
OLMo-3-7B & MATH500 random200 & 0.001 & 100 & 78.50 & 87/100 & 80.31 & 83.50 & yes \\
OLMo-3-7B & MATH500 random200 & 0.002 & 100 & 78.50 & 83/100 & 80.07 & 83.00 & yes \\
OLMo-3-7B & AMC12 train360 & 0.0005 & 100 & 75.56 & 6/100 & 72.89 & 76.67 & yes \\
OLMo-3-7B & AMC12 train360 & 0.001 & 100 & 75.56 & 6/100 & 72.90 & 76.67 & yes \\
OLMo-3-7B & AMC12 train360 & 0.002 & 100 & 75.56 & 0/100 & 71.01 & 75.56 & yes \\
OLMo-3-7B & GPQA train200 & 0.0005 & 100 & 33.00 & 9/100 & 30.26 & 36.00 & yes \\
OLMo-3-7B & GPQA train200 & 0.001 & 100 & 33.00 & 17/100 & 30.71 & 36.50 & yes \\
OLMo-3-7B & GPQA train200 & 0.002 & 100 & 33.00 & 13/100 & 30.46 & 37.00 & yes \\
OLMo-3-7B & MMLU-Redux train200 & 0.0005 & 100 & 50.00 & 4/100 & 48.09 & 52.00 & yes \\
OLMo-3-7B & MMLU-Redux train200 & 0.001 & 100 & 50.00 & 8/100 & 47.92 & 52.50 & yes \\
OLMo-3-7B & MMLU-Redux train200 & 0.002 & 100 & 50.00 & 9/100 & 47.16 & 53.00 & yes \\
\bottomrule
\end{tabular}
}
\label{tab:seed_profiling_details}
\end{table}

\subsection{Prompt Templates}
\label{app:prompts}

\textbf{Initial Answer Prompt.}
\phantomsection
\label{app:prompt_initial_answer}
For MATH500, the runs use the following system instruction and append the same
answer-format requirement to the user question:
\begin{promptbox}{MATH500 Initial Prompt}
\small
\textbf{System message.}
Let's think step by step. End with exactly one final answer in the form
\texttt{\textbackslash boxed\{your\_answer\}} and nothing else after the boxed answer.

\medskip
\textbf{User message.}
\emph{[Problem text]}

Let's think step by step. End with exactly one final answer in the form
\texttt{\textbackslash boxed\{your\_answer\}} and nothing else after the boxed answer.
\end{promptbox}
For AMC12, the answer space is A--E:
\begin{promptbox}{AMC12 Initial Prompt}
\small
\textbf{System message.}
Let's think step by step. End with exactly one final answer tag in the form
\texttt{<answer>X</answer>}, where $X$ is A, B, C, D, or E. Do not write
anything after the final answer tag.

\medskip
\textbf{User message.}
\emph{[Problem text]}

Return the single best answer as \texttt{<answer>X</answer>} where $X$ is one
of A, B, C, D, E.
\end{promptbox}
For GPQA and MMLU-Redux, the answer space is A--D:
\begin{promptbox}{GPQA and MMLU-Redux Initial Prompt}
\small
\textbf{System message.}
Let's think step by step. End with exactly one final answer tag in the form
\texttt{<answer>X</answer>}, where $X$ is A, B, C, or D. Do not write anything
after the final answer tag.

\medskip
\textbf{User message.}
\emph{[Problem text]}

Think briefly and solve the multiple-choice question. On the final line, write
exactly one answer in the form \texttt{\textbackslash boxed\{X\}}, where $X$ is
one of A, B, C, D. Do not write anything after the boxed answer.
\end{promptbox}
For multiple-choice datasets, the extractor maps answer tags, boxed letters,
parenthesized letters, and trailing single-letter answers to the same option key.

\textbf{Verification Debate Prompt.}
\phantomsection
\label{app:prompt_verify_then_answer}
The verification prompt is appended after the agent's previous response. It
includes clipped peer responses and then asks:
\begin{promptbox}{Verification Debate Prompt}
\small
Other assistants answered the same question. Verify your previous answer
against the peer answers below. First write a short verification note that
identifies the single most important agreement or disagreement that affects
correctness. Then, on the final line only, give your final answer in the same
required format.
\end{promptbox}
This is the main debate prompt used in the experiments.

\textbf{Standard Debate Prompt.}
\phantomsection
\label{app:prompt_standard_debate}
The standard prompt-control ablation replaces the verification instruction with:
\begin{promptbox}{Standard Debate Prompt}
\small
Other assistants answered the same question. Compare their solutions with
yours. Revise your answer if needed. Return only your final solution in the same
required format.
\end{promptbox}

\textbf{Tie-Break Judge Prompt.}
\phantomsection
\label{app:prompt_tiebreak_judge}
The judge receives the final candidate responses only for tied final votes. In
the verification-style runs, it is asked:
\begin{promptbox}{Tie-Break Judge Prompt}
\small
Review the candidate solutions below and verify which answer is best
supported. First write a short verification note explaining the decisive
evidence. Then, on the final line only, give the final answer in the same
required format.
\end{promptbox}

\clearpage
\subsection{Controlled Readout Probe Details}
\label{app:controlled_readout_probes}
For the answer-stripped rationale intervention, let $T$, $F$, and $N$ denote
the correct-rationale, wrong-rationale, and neutral contexts. We measure the
gold answer against all wrong-answer probability mass,
\[
m_a(x)=\log p(y_g\mid x)-\log\sum_{y\ne y_g}p(y\mid x).
\]
The correct-over-wrong gap is $m_a(T)-m_a(F)$. We also report the
help-over-harm asymmetry
\[
A=[m_a(T)-m_a(N)]-[m_a(N)-m_a(F)],
\]
which tests whether the helpful shift from correct reasoning exceeds the
reverse harmful shift from wrong reasoning.

\begin{table}[H]
\centering
\caption{Original peer-response effects and equivalent conclusion-only support
over recovery cases. Trace effects are means with 95\% confidence intervals;
$\widehat W$ is the median interpolated equivalent support. The final column
reports the correct-support minus wrong-support answer margin under the Verify
prompt at each dose.}
\label{tab:lvd_equivalent_support_full}
\small
\resizebox{\textwidth}{!}{%
\begin{tabular}{llccc}
\toprule
\textbf{Backbone} & \textbf{Dataset}
& $\boldsymbol{\Delta m_{\mathrm{trace}}}$ \textbf{[95\% CI]}
& \textbf{Median $\boldsymbol{\widehat W}$}
& \textbf{Correct$-$wrong margin, $\boldsymbol{d=1/2/3}$} \\
\midrule
Qwen3-4B & GPQA       & 8.15 [4.96, 11.41]  & 2.23 & 2.59 / 5.94 / 13.63 \\
Qwen3-4B & MMLU-Redux & 11.40 [9.02, 13.97] & $>3$ & 1.03 / 5.36 / 11.76 \\
OLMo3-7B & GPQA       & 4.57 [3.65, 5.50]   & 1.92 & 3.33 / 5.06 / 5.79 \\
OLMo3-7B & MMLU-Redux & 5.47 [4.52, 6.43]   & $>3$ & 3.59 / 5.18 / 5.53 \\
\bottomrule
\end{tabular}%
}
\end{table}

\begin{table}[H]
\centering
\caption{Truth sensitivity of answer-stripped trace rationales. Entries are
mean margin changes with 95\% confidence intervals, shown as Standard / Verify.}
\label{tab:lvd_answer_stripped_full}
\small
\resizebox{\textwidth}{!}{%
\begin{tabular}{lcc}
\toprule
\textbf{Backbone / dataset} & \textbf{Help-over-harm asymmetry $\boldsymbol{A}$}
& \textbf{Correct-over-wrong gap $\boldsymbol{T-F}$} \\
\midrule
Qwen3-4B / GPQA
& 7.34 [2.14, 12.45] / 6.84 [1.29, 12.33]
& 1.90 [.54, 3.39] / 4.59 [2.45, 6.96] \\
Qwen3-4B / MMLU-Redux
& 8.70 [3.50, 13.72] / 6.36 [2.34, 10.42]
& 5.98 [3.78, 8.32] / 5.50 [3.46, 7.69] \\
OLMo3-7B / GPQA
& 1.92 [.46, 3.35] / 1.39 [.17, 2.58]
& 1.70 [1.16, 2.31] / 2.01 [1.35, 2.77] \\
OLMo3-7B / MMLU-Redux
& 2.84 [.87, 4.85] / 2.24 [.78, 3.78]
& 1.93 [1.26, 2.63] / 2.10 [1.41, 2.83] \\
\bottomrule
\end{tabular}%
}
\end{table}

\begin{table}[H]
\centering
\caption{Free-generation accuracy after three peer messages. Each entry
evaluates the same 50 questions. Correct support is bolded.}
\label{tab:lvd_free_generation_full}
\small
\resizebox{\textwidth}{!}{%
\begin{tabular}{lllrrrr}
\toprule
\textbf{Backbone} & \textbf{Dataset} & \textbf{Prompt}
& \textbf{Neutral} & \textbf{Correct support}
& \textbf{Answer withheld} & \textbf{Shuffled rationale} \\
\midrule
Qwen3-4B & GPQA       & Standard & .32 & \textbf{.50} & .32 & .28 \\
Qwen3-4B & GPQA       & Verify   & .32 & \textbf{.46} & .34 & .14 \\
Qwen3-4B & MMLU-Redux & Standard & .28 & \textbf{.58} & .26 & .28 \\
Qwen3-4B & MMLU-Redux & Verify   & .38 & \textbf{.42} & .32 & .26 \\
OLMo3-7B & GPQA       & Standard & .30 & \textbf{.80} & .30 & .30 \\
OLMo3-7B & GPQA       & Verify   & .30 & \textbf{.86} & .30 & .30 \\
OLMo3-7B & MMLU-Redux & Standard & .30 & \textbf{.80} & .30 & .28 \\
OLMo3-7B & MMLU-Redux & Verify   & .30 & \textbf{.88} & .30 & .20 \\
\midrule
\textbf{Average} & & & .313 & \textbf{.663} & .305 & .255 \\
\bottomrule
\end{tabular}%
}
\end{table}

\clearpage

\textbf{Prompt-Control Tables.}
\phantomsection
\label{app:prompt_control_tables}
Tables~\ref{tab:r5_prompt_style_dynamics} and
\ref{tab:r5_prompt_style_control} report the completed control that replaces
the verification debate prompt with the normal \texttt{standard} debate prompt.
The control uses the same five datasets, seed banks, verifier settings, debate
rounds, saved traces, and finalizer as the main five-round runs. The standard
prompt is slightly slower in the mean vote trajectory: it trails
\texttt{verify\_then\_answer} in R1--R3, reaches its mean vote peak one round
later, and ends essentially tied at R5. It preserves more proposal-hit coverage
after R2, but the finalizer does not convert that extra coverage into higher
final accuracy.

These tables provide a secondary prompt-style analysis for the controlled
readout probes in \secref{subsec:controlled_readout_probes}. In the averaged
dynamics, the standard prompt retains 2.37 points more R5 proposal-hit coverage
than the verification prompt, while their final accuracies differ by only 0.07
points overall. The aggregate comparison therefore does not identify a uniform
prompt advantage or directly estimate verification mass.

The effect is not uniform across datasets or selection methods. The standard
prompt helps Vanilla Debate on MATH500, OlympiadBench, and GPQA, while the
verification prompt is better for MMLU-Redux Vanilla Debate and for
AC-Greedy Debate on GPQA and MMLU-Redux. We therefore use this control as a
secondary configuration diagnostic rather than evidence that one prompt
dominates the other. The controlled fixed-proposal interventions in
\secref{subsec:controlled_readout_probes} provide the direct test of
answer-specific evidence use.

\begin{table}[t]
\centering
\caption{Round-level prompt-style dynamics averaged over the 15 matched
dataset-method cells. R0 is identical because both controls use the same seed
answers before debate. $\Delta$ is Standard minus Verify. All values are
percentages.}
\small
\begin{tabular}{lrrrrrr}
\toprule
Round & Verify Vote & Standard Vote & $\Delta$ Vote & Verify Hit & Standard Hit & $\Delta$ Hit \\
\midrule
R0 & 76.04 & 76.04 & +0.00 & 83.95 & 83.95 & +0.00 \\
R1 & 77.95 & 77.81 & -0.14 & 85.38 & 84.89 & -0.49 \\
R2 & 78.19 & 78.12 & -0.08 & 82.20 & 83.38 & +1.18 \\
R3 & 78.37 & 78.16 & -0.21 & 80.44 & 82.98 & +2.54 \\
R4 & 78.10 & 78.43 & +0.32 & 79.56 & 82.04 & +2.48 \\
R5 & 78.21 & 78.13 & -0.08 & 79.26 & 81.62 & +2.37 \\
\bottomrule
\end{tabular}
\label{tab:r5_prompt_style_dynamics}
\end{table}

\begin{table}[t]
\centering
\caption{Prompt-style control for the five-round new-seed debate runs. Verify
Final is the completed \texttt{verify\_then\_answer} debate prompt; Standard
Final repeats the same datasets, seed banks, verifier settings, and
\texttt{vote\_then\_judge\_tiebreak} finalizer with the normal
\texttt{standard} debate prompt. $\Delta$ is Standard Final minus Verify Final.
All accuracy values are percentages.}
\small
\resizebox{\textwidth}{!}{%
\begin{tabular}{llrrrrrr}
\toprule
Dataset & Method & N & Verify Final & Standard Final & $\Delta$ & Standard R5 Vote & Standard R5 Hit \\
\midrule
AMC12 & Vanilla Debate & 739 & 92.96 & 92.42 & -0.54 & 92.29 & 93.78 \\
 & AC-Greedy Debate & 739 & 91.61 & 91.34 & -0.27 & 91.75 & 94.18 \\
 & SC-Greedy Debate & 739 & 91.88 & 92.02 & +0.14 & 91.88 & 93.50 \\
\midrule
MATH500 & Vanilla Debate & 300 & 91.33 & 92.00 & +0.67 & 91.67 & 93.00 \\
 & AC-Greedy Debate & 300 & 91.67 & 92.67 & +1.00 & 93.33 & 94.00 \\
 & SC-Greedy Debate & 300 & 92.00 & 91.67 & -0.33 & 91.67 & 93.33 \\
\midrule
OlympiadBench & Vanilla Debate & 474 & 81.22 & 83.76 & +2.53 & 83.97 & 85.86 \\
 & AC-Greedy Debate & 474 & 81.43 & 81.86 & +0.42 & 81.65 & 84.39 \\
 & SC-Greedy Debate & 474 & 80.80 & 81.22 & +0.42 & 81.43 & 83.33 \\
\midrule
GPQA & Vanilla Debate & 346 & 51.73 & 52.31 & +0.58 & 51.45 & 60.98 \\
 & AC-Greedy Debate & 346 & 52.02 & 49.71 & -2.31 & 50.58 & 60.98 \\
 & SC-Greedy Debate & 346 & 50.58 & 50.58 & +0.00 & 50.29 & 58.96 \\
\midrule
MMLU-Redux & Vanilla Debate & 300 & 72.33 & 70.00 & -2.33 & 70.67 & 73.33 \\
 & AC-Greedy Debate & 300 & 75.67 & 74.33 & -1.33 & 74.33 & 76.00 \\
 & SC-Greedy Debate & 300 & 75.33 & 75.67 & +0.33 & 75.00 & 78.67 \\
\midrule
Average & Vanilla Debate & -- & 77.92 & 78.10 & +0.18 & 78.01 & 81.39 \\
Average & AC-Greedy Debate & -- & 78.48 & 77.98 & -0.50 & 78.33 & 81.91 \\
Average & SC-Greedy Debate & -- & 78.12 & 78.23 & +0.11 & 78.05 & 81.56 \\
Overall & All & -- & 78.17 & 78.10 & -0.07 & 78.13 & 81.62 \\
\bottomrule
\end{tabular}
}
\label{tab:r5_prompt_style_control}
\end{table}

\clearpage
\subsection{Dynamics and Conversion Controls}
\label{app:dynamics_conversion_controls}
\paragraph{Immediate-finalizer control.}
A shrinking proposal-hit rate does not distinguish correct minorities that
were recovered from those that were simply lost. We therefore apply the same
\texttt{vote\_then\_judge\_tiebreak} finalizer independently to the proposals
saved at R0 through R5. This changes only the stage at which the finalizer is
applied; no debate generation is rerun. The analysis covers four societies and
five benchmarks, with equal weight assigned to each benchmark.

\begin{table}[H]
\centering
\caption{Immediate-finalizer and interaction control. Panel (a) compares the
R0 and R5 states by method. The final column is R5 Vote+Judge minus R0
Vote+Judge; confidence intervals resample questions within each benchmark
before recomputing the unweighted macro average. Panel (b) applies the same
finalizer at every saved stage and averages over all four societies and five
benchmarks. All accuracy, hit, gain, tie, and unanimity values are percentages.}
\label{tab:immediate_finalizer_control}
\vspace{0.3em}
\textbf{(a) Method-specific R0 and R5 comparison.}\\[0.2em]
\scriptsize
\resizebox{\textwidth}{!}{%
\begin{tabular}{lrrrrrrrr}
\toprule
Method & Vote R0 $\to$ R5 & Hit R0 $\to$ R5 & R0 V+J & R0 judge & R5 V+J & Additional R5 gain [95\% CI] & Tie R0 $\to$ R5 & Unanimous R0 $\to$ R5 \\
\midrule
Vanilla Debate & 75.65 $\to$ 77.99 & 83.03 $\to$ 78.98 & 76.03 & $+0.39$ & 77.89 & $+1.85\ [0.54,3.18]$ & 7.27 $\to$ 0.83 & 70.50 $\to$ 95.86 \\
Top-4 Debate & 76.25 $\to$ 78.16 & 84.11 $\to$ 79.27 & 76.84 & $+0.59$ & 78.42 & $+1.58\ [0.37,2.79]$ & 8.90 $\to$ 0.78 & 67.78 $\to$ 95.84 \\
AC-Greedy Debate & 76.84 $\to$ 78.41 & 84.93 $\to$ 79.60 & 77.59 & $+0.75$ & 78.51 & $+0.92\ [-0.39,2.27]$ & 9.58 $\to$ 0.91 & 64.61 $\to$ 95.45 \\
SC-Greedy Debate & 75.64 $\to$ 78.22 & 83.88 $\to$ 79.19 & 75.33 & $-0.31$ & 78.08 & $+2.74\ [1.40,4.11]$ & 9.56 $\to$ 0.63 & 63.07 $\to$ 96.11 \\
\bottomrule
\end{tabular}%
}

\vspace{0.7em}
\textbf{(b) Full stage trajectory.}\\[0.2em]
\small
\begin{tabular}{lrrrrrrr}
\toprule
Stage & Vote & Hit & Headroom & Vote+Judge & Gain over R0 vote & Tie & Unanimous \\
\midrule
R0 & 76.09 & 83.99 & 7.90 & 76.45 & $+0.36$ & 8.83 & 66.49 \\
R1 & 78.09 & 85.32 & 7.23 & 77.99 & $+1.89$ & 6.74 & 72.13 \\
R2 & 78.32 & 82.20 & 3.88 & 78.34 & $+2.25$ & 3.64 & 86.14 \\
R3 & 78.35 & 80.56 & 2.21 & 78.52 & $+2.42$ & 1.96 & 91.71 \\
R4 & 78.10 & 79.58 & 1.48 & 78.12 & $+2.02$ & 1.45 & 94.30 \\
R5 & 78.20 & 79.26 & 1.06 & 78.22 & $+2.13$ & 0.79 & 95.81 \\
\bottomrule
\end{tabular}
\end{table}

At R0, the finalizer contributes only 0.36 points on average. The R2 and R3
states reach 78.34\% and 78.52\% after the same finalizer, respectively 1.89
and 2.07 points above the R0 Vote+Judge result. This rise occurs while ties
fall from 8.83\% to 3.64\% at R2 and 1.96\% at R3, and unanimity rises from
66.49\% to 86.14\% and 91.71\%. The trajectory therefore shows rapid
interaction gains followed by a plateau, not a claim that the final round is
always optimal. At the method level, the additional R5 gain beyond the R0
readout ranges from 0.92 to 2.74 points; its interval excludes zero for
Vanilla, Top-4, and SC-Greedy.

\paragraph{Backbone-dependent conversion.}
We next ask whether initial headroom predicts final lift on OLMo as it does on
Qwen. For the pooled comparison, we use the four datasets and four society
constructions available for both backbones. The MMLU analysis pools the five
predefined Table~\ref{tab:main_results} subjects, and its Pearson interval is a
paired question-bootstrap interval.

\begin{table}[H]
\centering
\caption{Association between initial headroom $H=P_0-V_0$ and final lift
$G=F-V_0$ on the two backbones. The MMLU row reports a 95\% paired
question-bootstrap interval for Pearson $r$.}
\label{tab:backbone_headroom_conversion}
\small
\begin{tabular}{lcc}
\toprule
Scope & Qwen3-4B & OLMo3-7B \\
\midrule
Four shared datasets, pooled & $r=.770,\ \rho=.787$ & $r=.548,\ \rho=.648$ \\
Five MMLU subjects, pooled & $r=.568\ [.125,.787]$ & $r=.635\ [.219,.805]$ \\
\bottomrule
\end{tabular}
\end{table}

The positive OLMo associations show that proposal supply remains informative,
especially on MMLU. The failure mode is nevertheless dataset-dependent. On
OLMo MATH500, all four societies retain 3.33--6.33 points of initial headroom,
yet every final lift is negative, ranging from $-1.67$ to $-0.33$ points. This
localizes the failure to conversion rather than missing proposals: useful
answers are surfaced, but the backbone and readout do not reliably turn them
into a better final decision.

\clearpage
\subsection{Secondary Construction and Transfer Diagnostics}
\label{app:secondary_construction_diagnostics}
\paragraph{Transfer decomposition.}
The transfer experiment in Table~\ref{tab:amc12_mmlu_transfer_r5} can be read
through the exact gain decomposition in
Proposition~\ref{prop:recovery_damage_decomposition}.
Recovered, New correct, and Damage below are percentage-point contributions and
satisfy $\text{Gain}=\text{Recovered}+\text{New correct}-\text{Damage}$.

\begin{table}[H]
\centering
\caption{AMC12-to-MMLU-Redux transfer decomposition. Recovery rate is
conditional on R0 headroom; the contribution columns are unconditional
percentage points.}
\label{tab:transfer_decomposition}
\small
\resizebox{\textwidth}{!}{%
\begin{tabular}{llrrrrrrrr}
\toprule
Backbone & Method & R0 vote & Headroom & Recovery rate & Recovered & New correct & Damage & Gain & Final \\
\midrule
Qwen3-4B & SC-Greedy & 72.33 & 11.67 & 45.71 & 5.33 & 0.67 & 2.00 & $+4.00$ & 76.33 \\
Qwen3-4B & AC-Greedy & 75.33 & 8.33 & 48.00 & 4.00 & 0.33 & 3.00 & $+1.33$ & 76.67 \\
\midrule
OLMo3-7B & SC-Greedy & 55.33 & 14.67 & 36.36 & 5.33 & 0.67 & 3.00 & $+3.00$ & 58.33 \\
OLMo3-7B & AC-Greedy & 54.00 & 15.33 & 26.09 & 4.00 & 0.33 & 2.00 & $+2.33$ & 56.33 \\
\bottomrule
\end{tabular}%
}
\end{table}

On Qwen, AC-Greedy begins with the higher vote and therefore retains the best
final accuracy. SC-Greedy has more headroom, recovers 5.33 rather than 4.00
points, and incurs less damage, explaining its larger lift. On OLMo, the two
methods have similar headroom, but SC-Greedy has the higher recovery rate and
again recovers 5.33 rather than 4.00 points. Transfer therefore depends on both
proposal supply and its subsequent conversion.

\paragraph{Complementary-coverage diagnostics.}
We use expertise in a behavioral, task-level sense: an agent is useful when it
solves construction examples that its teammates miss and this complementarity
improves coverage on unseen questions. Let $C_1$ and $C_2$ be the construction
questions correctly solved by the first two selected agents. Lower
$|C_1\cap C_2|/|C_1\cup C_2|$ means less redundant coverage.

\begin{table}[H]
\centering
\caption{Complementarity and held-out coverage on Qwen3-4B MATH500. Panel (a)
uses the 200-question construction split. Panel (b) freezes each four-agent
society and evaluates cumulative R0 proposal-hit coverage on 300 held-out
questions; agents are listed in selection order, and tied construction
accuracies share the minimum rank.}
\label{tab:complementary_coverage}
\vspace{0.3em}
\textbf{(a) Correct-solution overlap of the first two agents.}\\[0.2em]
\small
\begin{tabular}{lrrr}
\toprule
Method & Intersection & Union & IoU $\downarrow$ \\
\midrule
Top-4 Debate & 182 & 188 & 96.8\% \\
AC-Greedy & 177 & 190 & 93.2\% \\
SC-Greedy & 176 & 186 & 94.6\% \\
\bottomrule
\end{tabular}

\vspace{0.7em}
\textbf{(b) Team composition and held-out cumulative R0 coverage.}\\[0.2em]
\scriptsize
\resizebox{\textwidth}{!}{%
\begin{tabular}{llrr}
\toprule
Method & Construction rank (accuracy), in team order & Held-out cumulative coverage & R5 final \\
\midrule
Top-4 Debate & 1 (92.5\%), 1 (92.5\%), 3 (92.0\%), 3 (92.0\%) & 88.33 $\to$ 91.33 $\to$ 92.00 $\to$ 92.67 & 90.67 \\
AC-Greedy & 1 (92.5\%), 25 (91.0\%), 25 (91.0\%), 8 (91.5\%) & 90.00 $\to$ 92.67 $\to$ 93.33 $\to$ 94.33 & 91.67 \\
SC-Greedy & 8 (91.5\%), 181 (89.5\%), 8 (91.5\%), 25 (91.0\%) & 90.67 $\to$ 92.67 $\to$ 93.67 $\to$ 93.67 & 92.00 \\
\bottomrule
\end{tabular}%
}
\end{table}

Both coverage objectives select less-overlapping first pairs than Top-4. When
the four-agent societies are evaluated out of sample, AC-Greedy covers 283 of
300 questions, SC-Greedy covers 281, and Top-4 covers 278. The selected agent
identities need not encode fixed semantic roles; the evidence supports the
narrower claim of transferable, task-level complementary proposal behavior.

\paragraph{Pool- and construction-size sensitivity.}
We run 200 paired subsampling replications across four Qwen benchmarks, with
four selected agents and 50 held-out questions per replication. Entries in
Table~\ref{tab:construction_sensitivity} are equal-dataset macro held-out
proposal-hit@4.

\begin{table}[H]
\centering
\caption{Sensitivity to candidate-pool size $N$ and construction-set size $m$.
Each arrow follows the increasing values shown in the column header.}
\label{tab:construction_sensitivity}
\small
\resizebox{\textwidth}{!}{%
\begin{tabular}{lcc}
\toprule
Selector & Pool $N=50\to100\to200\to500$ & Construction $m=25\to50\to100\to150$ \\
\midrule
AC-Greedy & $84.06\to84.08\to84.29\to84.12$ & $83.22\to83.59\to84.12\to84.38$ \\
SC-Greedy & $83.95\to84.03\to84.01\to83.81$ & $83.30\to83.36\to83.81\to83.85$ \\
\bottomrule
\end{tabular}%
}
\end{table}

Coverage is effectively saturated over $N=50$--$500$, while increasing $m$
from 25 to 150 improves AC-Greedy by 1.16 points (95\% CI $[0.87,1.43]$) and
SC-Greedy by 0.56 points ($[0.26,0.86]$). The useful perturbation scale remains
task- and backbone-dependent. We swept a broader range on construction data,
retained $\{5\times10^{-4},10^{-3},2\times10^{-3}\}$ because it contained a
useful concentration of high-quality seeds, and froze that band before test
evaluation; the per-scale profiles are in
Table~\ref{tab:seed_profiling_details}.

\paragraph{Hidden-Prefix construction.}
To test whether construction requires labels or discrete answer identities, we
remove final-answer tokens, keep the first 64 reasoning tokens, and record the
$\ell_2$-normalized final-layer state at the prefix endpoint. For each
question, Hidden-Prefix clusters the 500-agent pool with $k$-means ($k=4$, 20
restarts) in cosine geometry and selects four agents to maximize
inverse-support representation-cluster coverage. Neither answer identity nor
correctness enters selection.

Selection uses 200 construction questions and evaluation uses a disjoint
300-question MMLU-Redux test set. We fix one Hidden-Prefix society and one
TopAcc society and evaluate each with 10 paired generation seeds, temperature
$0.7$, top-$p=0.95$, four agents, five verification-prompt rounds, an
800-character peer context, and the same vote-then-judge finalizer.

\begin{table}[H]
\centering
\caption{Fixed-society Qwen3-4B MMLU-Redux results. Hidden-Prefix and TopAcc
report mean $\pm$ SD across 10 paired generation seeds; AC-Greedy and
SC-Greedy are the submitted fixed runs and are included only as reference
points.}
\label{tab:hidden_prefix_construction}
\small
\resizebox{\textwidth}{!}{%
\begin{tabular}{llcrrr}
\toprule
Selector & Construction signal & Requires & R0 hit@4 & R0 vote & R5 final \\
\midrule
AC-Greedy & Correct-answer coverage & Labels & 84.67 & 74.67 & 75.67 \\
SC-Greedy & Answer-cluster coverage & Clusters & 82.00 & 68.67 & 75.33 \\
Hidden-Prefix & Representation-cluster coverage & Neither & $84.30\pm0.69$ & $75.07\pm1.03$ & $77.00\pm1.25$ \\
TopAcc & Individual accuracy & Labels & $83.60\pm0.62$ & $73.93\pm1.07$ & $75.40\pm0.80$ \\
\bottomrule
\end{tabular}%
}
\end{table}

Hidden-Prefix reaches 84.30 hit@4 without labels or answer identities, within
0.37 points of supervised AC-Greedy. Against its paired TopAcc control, it
improves R0 vote by 1.13 points (95\% CI $[0.20,2.07]$) and R5 final accuracy
by 1.60 points ($[0.32,2.88]$); a two-way seed-and-question interval for the
final gain is $[0.07,3.10]$. This establishes an answer-string-free
construction route on a multiple-choice task. It does not yet validate
semantic equivalence classes or evaluation for open-ended generation.

\clearpage
\subsection{Complete Round-Level Statistics}
\label{app:round_level_statistics}

\begin{figure}[t]
\centering
\includegraphics[width=\textwidth]{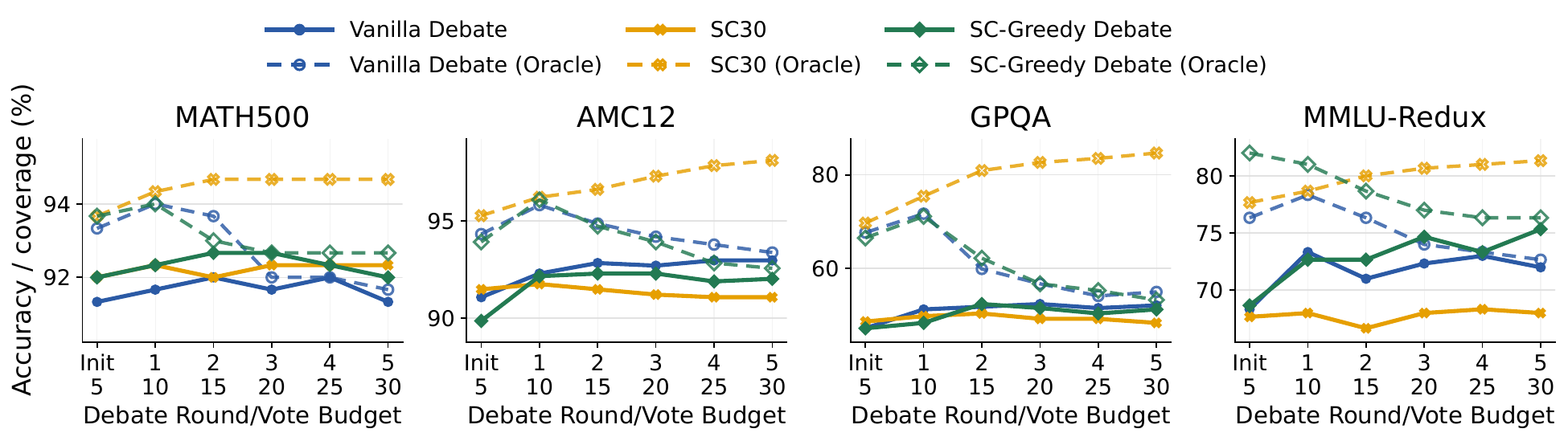}
\caption{Label-free five-round dynamics. This appendix counterpart to
\figref{fig:r5_debate_dynamics_label_based} shows methods that do not use
labeled training diagnostics for society selection. Solid lines show
round-level majority-vote accuracy before finalization; dashed hollow legend
entries marked ``Oracle'' show proposal-hit coverage.}
\label{fig:r5_debate_dynamics_label_free}
\end{figure}

This appendix collects the full round-level tables behind
\figref{fig:r5_debate_dynamics_label_based} and
\figref{fig:r5_debate_dynamics_label_free}. The tables expand the mechanism
diagnostics used in \secref{sec:experiments} and \secref{sec:supply}. They ask two
questions: whether the selected agents supply correct proposals that vote
misses, and whether the matched readout converts that supply into final
accuracy without excessive damage.

\paragraph{Metrics and Reading Guide.}
Vote is the majority-vote accuracy on the proposal multiset at a given round.
Proposal Hit is the probability that at least one visible proposal at that
round contains the correct answer. Headroom $H$ is the gap between these two
events: it counts cases where the proposal multiset contains the correct answer
but majority vote misses it. Final is the end-of-run
\texttt{vote\_then\_judge\_tiebreak} accuracy. Gain $G$ is Final minus Vote.
Recovery $R$ is the final accuracy on headroom instances. New Correctness $I$
is the final accuracy on instances where the visible proposal multiset did not
contain the correct answer and Vote was wrong. Damage $L$ is the failure rate
on instances where Vote is already correct. Thus Proposal Hit and Headroom
diagnose proposal supply, while Recovery, New Correctness, and Damage diagnose
readout conversion.

These definitions make the appendix a supply-readout analysis rather than a
second main-results table. Higher Proposal Hit or Headroom means that the
selected agents surface more correct proposals. Higher Recovery means that the
readout converts more surfaced correct minorities into final answers. Lower
Damage means that the readout preserves cases already solved by majority vote.
New Correctness tracks gains not explained by the proposal-hit supply at the
chosen round. Final accuracy combines these effects, so it can hide whether a
gain comes from better proposal supply, better readout conversion, or lower
damage.

The main pattern is that useful proposal supply appears early and then contracts
as debate moves toward consensus. At R0, all four debate societies have
7--8 points of recoverable headroom on average. By R5, the remaining headroom
falls to about one point. This supports the fixed-initial-proposal view in
\secref{sec:headroom_lvd}: the society often surfaces correct minorities before
interaction, and later rounds can reduce the set of alternatives even when vote
accuracy improves. Debate gains therefore depend on recovering useful minority
answers before proposal-hit coverage collapses.

The recovery and damage columns show why higher proposal-hit supply is not
enough by itself. For example, AC-Greedy Debate has the highest average R0
proposal-hit rate, 84.93\%, but its final lift over the R0 vote is 1.64 points.
SC-Greedy Debate has lower average R0 proposal-hit rate, 83.88\%, but larger
average lift, 2.48 points, because its starting vote is lower and more of its
gain comes from converting missed proposals. These differences are why the main
text reports headroom, recovery, and damage rather than final accuracy alone.

\subsubsection{Average Recovery Diagnostics}

\begin{table}[t]
\centering
\caption{Round-level recovery diagnostics averaged over the five benchmarks.
For each round, Vote and Hit are the round-specific majority-vote and
proposal-hit accuracies, while Final is the end-of-run
\texttt{vote\_then\_judge\_tiebreak} accuracy. We define $H=P-V$, $G=D-V$,
$\rho=G/H$, $R=\Pr(D=1\mid P=1,V=0)$,
$I=\Pr(D=1\mid P=0,V=0)$, and $L=\Pr(D=0\mid V=1)$, where $D$ is the final
decision. All values except $\rho$ are percentages; averages are unweighted
over benchmarks.}
\scriptsize
\resizebox{\textwidth}{!}{%
\begin{tabular}{llrrrrrrrrr}
\toprule
Method & Round & Vote & Hit & Final & $H$ & $G$ & $\rho$ & $R$ & $I$ & $L$ \\
\midrule
Vanilla Debate & R0 & 75.65 & 83.03 & 77.92 & 7.39 & 2.27 & 0.31 & 53.49 & 9.94 & 5.00 \\
 & R1 & 77.85 & 85.05 & 77.92 & 7.20 & 0.07 & 0.01 & 43.10 & 1.23 & 4.25 \\
 & R2 & 77.84 & 81.86 & 77.92 & 4.02 & 0.08 & 0.02 & 24.33 & 0.14 & 1.58 \\
 & R3 & 78.05 & 80.07 & 77.92 & 2.03 & -0.13 & -0.06 & 43.39 & 0.60 & 1.37 \\
 & R4 & 78.13 & 79.04 & 77.92 & 0.92 & -0.21 & -0.23 & 38.89 & 0.13 & 1.08 \\
 & R5 & 77.99 & 78.98 & 77.92 & 0.99 & -0.08 & -0.08 & 10.00 & 0.13 & 0.33 \\
\midrule
Top-4 Debate & R0 & 76.25 & 84.11 & 78.28 & 7.87 & 2.03 & 0.26 & 40.88 & 7.26 & 4.53 \\
 & R1 & 78.51 & 85.15 & 78.28 & 6.64 & -0.23 & -0.03 & 29.46 & 0.38 & 3.67 \\
 & R2 & 78.72 & 82.22 & 78.28 & 3.51 & -0.44 & -0.13 & 28.94 & 0.48 & 2.16 \\
 & R3 & 78.30 & 80.93 & 78.28 & 2.63 & -0.02 & -0.01 & 30.09 & 0.00 & 1.18 \\
 & R4 & 78.11 & 79.66 & 78.28 & 1.56 & 0.17 & 0.11 & 42.93 & 0.00 & 0.73 \\
 & R5 & 78.16 & 79.27 & 78.28 & 1.11 & 0.11 & 0.10 & 17.78 & 0.00 & 0.25 \\
\midrule
AC-Greedy Debate & R0 & 76.84 & 84.93 & 78.48 & 8.09 & 1.64 & 0.20 & 48.94 & 7.70 & 6.01 \\
 & R1 & 78.63 & 85.72 & 78.48 & 7.09 & -0.15 & -0.02 & 34.76 & 1.68 & 4.23 \\
 & R2 & 78.43 & 82.49 & 78.48 & 4.06 & 0.06 & 0.01 & 45.25 & 0.40 & 2.46 \\
 & R3 & 78.73 & 80.84 & 78.48 & 2.10 & -0.25 & -0.12 & 37.24 & 0.00 & 1.58 \\
 & R4 & 78.50 & 79.97 & 78.48 & 1.47 & -0.02 & -0.01 & 39.44 & 0.00 & 0.64 \\
 & R5 & 78.41 & 79.60 & 78.48 & 1.19 & 0.07 & 0.06 & 32.00 & 0.00 & 0.11 \\
\midrule
SC-Greedy Debate & R0 & 75.64 & 83.88 & 78.12 & 8.23 & 2.48 & 0.30 & 48.52 & 7.20 & 5.34 \\
 & R1 & 77.37 & 85.35 & 78.12 & 7.98 & 0.75 & 0.09 & 34.76 & 1.45 & 3.55 \\
 & R2 & 78.32 & 82.25 & 78.12 & 3.93 & -0.20 & -0.05 & 25.29 & 0.00 & 2.58 \\
 & R3 & 78.33 & 80.42 & 78.12 & 2.09 & -0.21 & -0.10 & 29.56 & 0.00 & 1.37 \\
 & R4 & 77.69 & 79.65 & 78.12 & 1.97 & 0.43 & 0.22 & 24.71 & 0.00 & 0.42 \\
 & R5 & 78.22 & 79.19 & 78.12 & 0.97 & -0.10 & -0.10 & 6.67 & 0.00 & 0.26 \\
\bottomrule
\end{tabular}
}
\label{tab:round_recovery_diagnostics}
\end{table}

\paragraph{Interpretation.}
\tabref{tab:round_recovery_diagnostics} separates proposal supply from readout
conversion. Initial headroom is substantial for every method, which means the
selected societies surface correct answers that majority vote does not extract.
However, recovery falls sharply after the first few rounds because debate
compresses the proposal set. Damage also falls over rounds, since fewer
vote-correct cases are changed once the agents converge. New Correctness is
largest at R0 and becomes small after the debate trace has mostly converged,
which means most final gains are explained by surfaced proposal supply rather
than late invention of new correct answers. The tradeoff is
therefore not monotone: later rounds can make the majority more stable, but
they can also remove the minority alternatives that the readout needs to
recover.

\subsubsection{Five-Round New-Seed Debate Results}

\begin{table}[t]
\centering
\caption{Five-round debate results for the new seed teams. Final is accuracy
after \texttt{vote\_then\_judge\_tiebreak}; Lift is Final minus Init Vote.
Bold marks the best final accuracy within each dataset. Average rows report the
unweighted mean across the five benchmarks.}
\small
\resizebox{\textwidth}{!}{%
\begin{tabular}{llrrrrrrrr}
\toprule
Dataset & Method & N & Init Vote & R5 Vote & Final & Init Hit & R5 Hit & Lift & Judge Ties \\
\midrule
MATH500 & Vanilla Debate & 300 & 91.33 & 91.33 & 91.33 & 93.33 & 91.67 & +0.00 & 1 \\
MATH500 & AC-Greedy Debate & 300 & 91.33 & 91.33 & 91.67 & 94.33 & 91.67 & +0.33 & 2 \\
MATH500 & SC-Greedy Debate & 300 & 92.00 & 92.00 & \textbf{92.00} & 93.67 & 92.67 & +0.00 & 0 \\
OlympiadBench & Vanilla Debate & 474 & 80.38 & 81.65 & 81.22 & 83.54 & 82.28 & +0.84 & 7 \\
OlympiadBench & AC-Greedy Debate & 474 & 79.11 & 81.43 & \textbf{81.43} & 82.91 & 81.86 & +2.32 & 7 \\
OlympiadBench & SC-Greedy Debate & 474 & 80.59 & 80.59 & 80.80 & 83.33 & 81.22 & +0.21 & 5 \\
AMC12 & Vanilla Debate & 739 & 91.07 & 92.96 & \textbf{92.96} & 94.32 & 93.37 & +1.89 & 4 \\
AMC12 & AC-Greedy Debate & 739 & 90.53 & 91.88 & 91.61 & 95.13 & 92.83 & +1.08 & 9 \\
AMC12 & SC-Greedy Debate & 739 & 89.85 & 92.02 & 91.88 & 93.91 & 92.56 & +2.03 & 2 \\
GPQA & Vanilla Debate & 346 & 47.11 & 52.02 & 51.73 & 67.63 & 54.91 & +4.62 & 4 \\
GPQA & AC-Greedy Debate & 346 & 48.55 & 51.73 & \textbf{52.02} & 67.63 & 54.62 & +3.47 & 3 \\
GPQA & SC-Greedy Debate & 346 & 47.11 & 51.16 & 50.58 & 66.47 & 53.18 & +3.47 & 4 \\
MMLU-Redux & Vanilla Debate & 300 & 68.33 & 72.00 & 72.33 & 76.33 & 72.67 & +4.00 & 2 \\
MMLU-Redux & AC-Greedy Debate & 300 & 74.67 & 75.67 & \textbf{75.67} & 84.67 & 77.00 & +1.00 & 1 \\
MMLU-Redux & SC-Greedy Debate & 300 & 68.67 & 75.33 & 75.33 & 82.00 & 76.33 & +6.67 & 2 \\
\midrule
Average & Vanilla Debate & -- & 75.65 & 77.99 & 77.92 & 83.03 & 78.98 & +2.27 & 3.6 \\
Average & AC-Greedy Debate & -- & 76.84 & 78.41 & \textbf{78.48} & 84.93 & 79.60 & +1.64 & 4.4 \\
Average & SC-Greedy Debate & -- & 75.64 & 78.22 & 78.12 & 83.88 & 79.19 & +2.48 & 2.6 \\
\bottomrule
\end{tabular}
}
\label{tab:r5_new_seed_debate}
\end{table}

\paragraph{Interpretation.}
\tabref{tab:r5_new_seed_debate} shows that the same mechanism appears at the
dataset level. On GPQA and MMLU-Redux, initial proposal-hit rates are much
higher than initial vote accuracy, and debate produces the largest lifts. On
MATH500 and AMC12, the initial vote is already high, so the available room for
final improvement is smaller even when proposal-hit coverage is high. This is
consistent with \propref{prop:recovery_damage_decomposition}: final gain is
largest when there is recoverable headroom and the readout converts it without
large damage.

The construction methods also behave differently across datasets. AC-Greedy
Debate gives the best average final accuracy, but SC-Greedy Debate gives the
largest average lift from its own initial vote. This supports the role of
label-free coverage as a practical supply constructor, while also showing why
the supervised construction remains a useful reference.

\section{Proofs for \secref{sec:headroom_lvd}}
\label{app:proofs_headroom_lvd}

\subsection{Proof of \propref{prop:recovery_damage_decomposition}}
\label{app:proof_recovery_damage_decomposition}

\begin{proof}[Proof of \propref{prop:recovery_damage_decomposition}]
By definition,
\[
    \mathcal G_\eta(S)=\mathbb E[D_\eta-V]
    =
    \Pr(D_\eta=1,V=0)-\Pr(D_\eta=0,V=1).
\]
We decompose the positive term by whether the initial proposal supply hit the
correct answer:
\[
\Pr(D_\eta=1,V=0)
=
\Pr(D_\eta=1,P=1,V=0)
+
\Pr(D_\eta=1,P=0,V=0).
\]
Using the definitions of $R_\eta$ and $I_\eta$ gives
\[
\Pr(D_\eta=1,P=1,V=0)
=
\Pr(P=1,V=0)\Pr(D_\eta=1\mid P=1,V=0)
=
\mathcal H(S)R_\eta(S),
\]
and
\[
\Pr(D_\eta=1,P=0,V=0)
=
\Pr(P=0,V=0)I_\eta(S).
\]
The negative term is
\[
\Pr(D_\eta=0,V=1)
=
\Pr(V=1)\Pr(D_\eta=0\mid V=1)
=
\Pr(V=1)L_\eta(S).
\]
Substituting these three identities proves
\eqnref{eq:headroom_decomposition}.
\end{proof}

\subsection{Proof of \lemref{lem:verification_weighted_conjugacy}}
\label{app:proof_verification_weighted_conjugacy}

\begin{lemma}[Verification-Weighted Dirichlet Conjugacy]
\label{lem:verification_weighted_conjugacy}
Let $\pi=(\pi^{(k)})_{k\in\mathcal Y_t}$ be an answer-probability vector over
$\mathcal Y_t$, and let
$\pi\mid x\sim\operatorname{Dirichlet}(\alpha_0)$ with
$\alpha_0^{(k)}=\kappa q_\theta^0(k\mid x)$. Conditional on an effective
verified pseudo-count vector $W_{\eta,t}$, use the verification-weighted
pseudo-likelihood
\[
    p(\mathcal F_t,W_{\eta,t}\mid \pi)
    \propto
    \prod_{k\in\mathcal Y_t}
    \left(\pi^{(k)}\right)^{C_t^{(k)}+W_{\eta,t}^{(k)}}.
\]
Then the posterior remains Dirichlet, and its posterior predictive is
\eqnref{eq:lvd_posterior}.
\end{lemma}

\begin{proof}[Proof of \lemref{lem:verification_weighted_conjugacy}]
The Dirichlet prior has density
\[
    p(\pi\mid x)
    \propto
    \prod_{k\in\mathcal Y_t}
    \left(\pi^{(k)}\right)^{\alpha_0^{(k)}-1}.
\]
Multiplying by the verification-weighted pseudo-likelihood gives
\[
\begin{aligned}
    p(\pi\mid x,\mathcal F_t,W_{\eta,t})
    &\propto
    p(\pi\mid x)\,
    p(\mathcal F_t,W_{\eta,t}\mid \pi)\\
    &\propto
    \prod_{k\in\mathcal Y_t}
    \left(\pi^{(k)}\right)^{
    \alpha_0^{(k)}-1+C_t^{(k)}+W_{\eta,t}^{(k)}
    }.
\end{aligned}
\]
This is the kernel of a Dirichlet distribution with parameters
\[
    \alpha_0^{(k)}+C_t^{(k)}+W_{\eta,t}^{(k)}
    =
    \kappa q_\theta^0(k\mid x)+C_t^{(k)}+W_{\eta,t}^{(k)}.
\]
The posterior predictive for a categorical draw from a Dirichlet posterior is
the posterior mean:
\[
    \mathbb E[\pi^{(k)}\mid x,\mathcal F_t,W_{\eta,t}]
    =
    \frac{
    \kappa q_\theta^0(k\mid x)+C_t^{(k)}+W_{\eta,t}^{(k)}
    }{
    \sum_{j\in\mathcal Y_t}
    \left(\kappa q_\theta^0(j\mid x)+C_t^{(j)}+W_{\eta,t}^{(j)}\right)
    }.
\]
Since $\sum_j q_\theta^0(j\mid x)=1$,
$\sum_j C_t^{(j)}=C_t^+$, and
$\sum_j W_{\eta,t}^{(j)}=W_{\eta,t}^+$, this is exactly
\eqnref{eq:lvd_posterior}.
\end{proof}

\subsection{Proof of \lemref{lem:lvd_marginalization}}
\label{app:proof_lvd_marginalization}

\begin{lemma}[LVD Marginalization]
\label{lem:lvd_marginalization}
Let
\begin{equation}
    q_\eta^v(z\mid x,\mathcal F_t)
    =
    \operatorname{softmax}_{z\in\mathcal Y_t}
    \bigl(a_\eta(z;x,\mathcal F_t)\bigr)
    \label{eq:verification_direction}
\end{equation}
be a soft allocation of verification support. If a latent verification draw
$v=z$ adds mass $\tau_{\eta,t}\ge 0$ to answer $z$, define
\begin{equation}
    p_\eta(y=k\mid v=z,x,\mathcal F_t)
    =
    \frac{
    \kappa q_\theta^0(k\mid x)+C_t^{(k)}
    +\tau_{\eta,t}\mathbf 1\{z=k\}
    }{
    \kappa+C_t^+ +\tau_{\eta,t}
    }.
    \label{eq:lvd_conditional}
\end{equation}
If $v\sim q_\eta^v(\cdot\mid x,\mathcal F_t)$, marginalizing $v$ gives
\eqnref{eq:lvd_posterior} with
$W_{\eta,t}^{(k)}=\tau_{\eta,t}q_\eta^v(k\mid x,\mathcal F_t)$.
\end{lemma}

\begin{proof}[Proof of \lemref{lem:lvd_marginalization}]
For fixed $v=z$, the conditional readout is
\[
    p_\eta(y=k\mid v=z,x,\mathcal F_t)
    =
    \frac{
    \kappa q_\theta^0(k\mid x)+C_t^{(k)}
    +\tau_{\eta,t}\mathbf 1\{z=k\}
    }{
    \kappa+C_t^+ +\tau_{\eta,t}
    }.
\]
Marginalizing over $v\sim q_\eta^v(\cdot\mid x,\mathcal F_t)$ gives
\[
\begin{aligned}
    p_\eta(y=k\mid x,\mathcal F_t)
    &=
    \sum_{z\in\mathcal Y_t}
    q_\eta^v(z\mid x,\mathcal F_t)
    p_\eta(y=k\mid v=z,x,\mathcal F_t)\\
    &=
    \frac{
    \kappa q_\theta^0(k\mid x)+C_t^{(k)}
    +\tau_{\eta,t}
    \sum_{z\in\mathcal Y_t}
    q_\eta^v(z\mid x,\mathcal F_t)\mathbf 1\{z=k\}
    }{
    \kappa+C_t^+ +\tau_{\eta,t}
    }\\
    &=
    \frac{
    \kappa q_\theta^0(k\mid x)+C_t^{(k)}
    +\tau_{\eta,t}q_\eta^v(k\mid x,\mathcal F_t)
    }{
    \kappa+C_t^+ +\tau_{\eta,t}
    }.
\end{aligned}
\]
Thus the marginal readout has the LVD form with
\[
    W_{\eta,t}^{(k)}=\tau_{\eta,t}q_\eta^v(k\mid x,\mathcal F_t),
    \qquad
    W_{\eta,t}^+=\tau_{\eta,t}.
\]
\end{proof}

\subsection{Proof of \propref{prop:verification_directed_drift}}
\label{app:proof_verification_directed_drift}

\begin{proof}[Proof of \propref{prop:verification_directed_drift}]
Let
$Q_t^{\mathrm{prop}}=\alpha_t^\star/A_t$ and
$\widetilde Q_{\eta,t}^{\mathrm{LVD}}=
(\alpha_t^\star+W_{\eta,t}^{(y^\star)})/(A_t+W_{\eta,t}^+)$.
We first compute the difference between the LVD correctness mass and the
proposal-only correctness mass:
\[
\begin{aligned}
    \widetilde Q_{\eta,t}^{\mathrm{LVD}}-Q_t^{\mathrm{prop}}
    &=
    \frac{\alpha_t^\star+W_{\eta,t}^{(y^\star)}}{A_t+W_{\eta,t}^+}
    -
    \frac{\alpha_t^\star}{A_t}\\
    &=
    \frac{
    A_t(\alpha_t^\star+W_{\eta,t}^{(y^\star)})
    -
    \alpha_t^\star(A_t+W_{\eta,t}^+)
    }{
    A_t(A_t+W_{\eta,t}^+)
    }\\
    &=
    \frac{
    A_tW_{\eta,t}^{(y^\star)}
    -
    \alpha_t^\star W_{\eta,t}^+
    }{
    A_t(A_t+W_{\eta,t}^+)
    }.
\end{aligned}
\]
If $W_{\eta,t}$ is fixed conditional on $\mathcal F_t$ and
$W_{\eta,t}^+>0$, then the denominator is positive, and positive drift is
equivalent to
\[
    A_tW_{\eta,t}^{(y^\star)}
    >
    \alpha_t^\star W_{\eta,t}^+.
\]
Dividing by $A_tW_{\eta,t}^+$ gives
\[
    \frac{W_{\eta,t}^{(y^\star)}}{W_{\eta,t}^+}
    >
    \frac{\alpha_t^\star}{A_t}
    =
    Q_t^{\mathrm{prop}},
\]
which proves the fixed-vector condition.
If $W_{\eta,t}$ is stochastic, the corresponding expected-drift condition is
\[
    \mathbb E\!\left[
    \frac{
    A_tW_{\eta,t}^{(y^\star)}
    -
    \alpha_t^\star W_{\eta,t}^+
    }{
    A_t(A_t+W_{\eta,t}^+)
    }
    \mid
    \mathcal F_t
    \right]
    >
    0.
\]
\end{proof}

\subsection{Proof of \propref{prop:correct_minority_recovery}}
\label{app:proof_correct_minority_recovery}

\begin{proof}[Proof of \propref{prop:correct_minority_recovery}]
The LVD posterior predictive uses the same denominator for every candidate
answer. Therefore LVD ranks $y^\star$ above $y_v$ if and only if its numerator
is larger:
\[
    \alpha_t^{\mathrm{prop}}(y^\star)+W_{\eta,t}^{(y^\star)}
    >
    \alpha_t^{\mathrm{prop}}(y_v)+W_{\eta,t}^{(y_v)}.
\]
Rearranging terms gives
\[
    W_{\eta,t}^{(y^\star)}-W_{\eta,t}^{(y_v)}
    >
    \alpha_t^{\mathrm{prop}}(y_v)-\alpha_t^{\mathrm{prop}}(y^\star),
\]
which is \eqnref{eq:lvd_recovery_condition}.
\end{proof}

\section{Limitations}\label{appsec:limitation}
The method is a headroom-recovery method, not a general claim that debate beats
vote. LVD can only recover answers that have already been surfaced by the agent
society. If the society contains no correct proposal, the readout has no
headroom to recover. LVD can also damage cases where vote is already correct,
which is why the experiments report damage rate in addition to final accuracy.

The label-free construction proxy is intentionally simple, but its reliability
depends on answer clustering. It is most natural in exact-answer or
near-exact-answer settings where valid answers can be normalized and grouped.
Open-ended tasks may require a different proxy and a different way to define
proposal-hit coverage. Appendix~\ref{app:secondary_construction_diagnostics}
shows that a Hidden-Prefix representation can construct a competitive society
without labels or answer identities. This is evidence for answer-string-free
construction, but it remains a multiple-choice experiment and does not by
itself validate semantic equivalence or evaluation for open-ended outputs.

The experiments use extra test-time compute. We therefore treat budget alignment
as part of the empirical claim, not an implementation detail. We report matched
or explicitly stated budgets for vote, plain debate, and LVD.

\end{document}